\documentclass{article}

\PassOptionsToPackage{numbers}{natbib}
\usepackage[preprint]{neurips_2026}
\usepackage[utf8]{inputenc}
\usepackage[T1]{fontenc}
\usepackage{hyperref}
\hypersetup{pdfborder={0 0 0}}
\usepackage{url}
\usepackage{booktabs}
\usepackage{amsfonts}
\usepackage{amsmath}
\usepackage{amssymb}
\usepackage{nicefrac}
\usepackage{microtype}
\usepackage[table]{xcolor}
\definecolor{tabfirst}{rgb}{1, 0.7, 0.7} %
\definecolor{tabsecond}{rgb}{1, 0.85, 0.7} %
\definecolor{tabthird}{rgb}{1, 1, 0.7} %
\definecolor{appendixrule}{HTML}{D6DDE8}
\usepackage{graphicx}
\usepackage{algorithm}
\usepackage{algorithmic}
\usepackage{subcaption}
\usepackage{multirow}
\usepackage{mathtools}
\usepackage{bm}
\usepackage{enumitem}
\usepackage{amsthm}
\usepackage{wrapfig}
\usepackage[most]{tcolorbox}

\DeclareMathOperator*{\argmin}{arg\,min}

\DeclareMathOperator{\KL}{KL}

\newcommand{\E}{\mathbb{E}}

\newcommand{\bx}{\mathbf{x}}
\newcommand{\bz}{\mathbf{z}}
\newcommand{\beps}{\boldsymbol{\epsilon}}

\newtheorem{theorem}{Theorem}
\newtheorem{proposition}{Proposition}
\newtheorem{remark}{Remark}

\newcommand{\AppendixTocSection}[2]{%
  \vspace{0.35em}%
  \noindent\hyperref[#1]{\textbf{~\ref*{#1}. #2}}%
  \nobreak\leaders\hbox{\textcolor{appendixrule}{.}}\hfill\nobreak
  \hyperref[#1]{\textbf{~\pageref*{#1}}}\par
}

\newcommand{\AppendixTocSubsection}[2]{%
  \noindent\hspace*{1.35em}%
  \hyperref[#1]{\small \ref*{#1}\quad #2}%
  \nobreak\leaders\hbox{\textcolor{appendixrule}{.}}\hfill\nobreak
  \hyperref[#1]{\small ~\pageref*{#1}}\par
}

\title{TILDE: TILt-based Distributional Erasure for Concept Unlearning}

\author{%
  Naveen George\thanks{Work done during an internship at Sony AI.}\\
  Indian Institute of Technology Hyderabad\\
  \texttt{ai23mtech12001@iith.ac.in}\\
  \And
  Naoki Murata\\
  Sony AI\\
  \texttt{naoki.murata@sony.com}\\
  \And
  Yuhta Takida\\
  Sony AI\\
  \AND
  Konda Reddy Mopuri\\
  Indian Institute of Technology Hyderabad\\
  \And
  Yuki Mitsufuji\\
  Sony AI \& Sony Group Corporation
}

\begin{document}

\maketitle

\begin{abstract}
Concept unlearning in text-to-image diffusion models is critical for safe and practical deployment: with rising privacy concerns, copyright disputes, trademark constraints, and safety regulations, deployed systems must be able to suppress unwanted concepts after training.
Existing methods often remove the target concept effectively, but practical unlearning also requires an equally fundamental property: the unlearned model should retain quality, diversity, and semantic coverage on benign generation.
The gold standard is a retain-only model trained from scratch without the unwanted data.
However, common erasure objectives do not specify which post-unlearning distribution should approximate this reference, leaving retention as an implicit consequence of the update rule.
We propose \textsc{TILDE}, TILt-based Distributional Erasure, which formulates concept unlearning as a distributional alignment problem: the desired target is the minimum-deviation conditional distribution from the pretrained model under a forgetting constraint.
This energy-tilted, anchor-free target suppresses concept-expressing images while preserving benign relative mass for each prompt.
We instantiate this principle with residual $\nabla$-GFlowNet training, which learns the score correction induced by the forget energy relative to the pretrained diffusion model.
Across objects, artistic styles, and characters, \textsc{TILDE} achieves strong forgetting while improving retention and distributional fidelity over prior baselines.
\end{abstract}

\section{Introduction}
\label{sec:intro}

Text-to-image diffusion models have rapidly become part of everyday creative and professional workflows, enabling high-fidelity image creation and editing across artistic styles, fictional characters, personal identities, products, and common scenes~\citep{rombach2022high}.
As these models become more widely used, the ability to edit their behavior after training is essential for legal compliance, privacy protection, and safety.
Privacy regulations such as the GDPR recognize rights to erasure for personal data~\citep{gdpr2016}, while copyright, trademark, identity, and safety constraints can require suppressing specific concepts after training.
Concept unlearning addresses this requirement by editing a pretrained text-to-image model so that it no longer generates images expressing an undesirable concept.

The central challenge is that unlearning requires not only effective forgetting but equally effective preservation: the edited model should maintain image quality and prompt alignment on benign prompts, without degrading related or general concepts.
An edited model that successfully forgets but degrades benign generation does not approximate unlearning; it solves the forgetting problem by introducing a retention failure.
The ideal reference is a retain-only model trained from scratch without ever seeing the unwanted data: the target concept is absent, while the rest of the generative distribution remains intact~\citep{cho2025fade,zhao2024what}.
We study concept unlearning as an efficient approximation to this retain-only behavior.

\paragraph{Desiderata.}
This view yields three desiderata.
\textbf{(D1) Effective forgetting:} the post-unlearning model should assign low probability to images that express the target concept.
\textbf{(D2) Local preservation:} benign prompts that are semantically close to the forgotten concept should remain close to pretrained behavior, because this is where collateral damage is most likely.
\textbf{(D3) Distributional fidelity:} outside the forget region, the post-unlearning distribution should stay close to the pretrained or retain-only reference at the per-prompt level, not just match coarse aggregate scores such as FID~\citep{cho2025fade}.

Existing methods use erasure proxies such as score suppression~\citep{gandikota2023erasing,huang2024receler}, anchor-based editing~\citep{kumari2023ablating,gandikota2024unified,bui2025fantastic}, preference or reward-based optimization~\citep{park2024duo}, influence or trajectory-level steering~\citep{wu2024scissorhands,li2025set,wu2025erasing}, or trajectory-level GFlowNet objectives~\citep{kusumba2025eraseflow}.
These objectives can suppress the forget concept, but they generally do not specify where the resulting distribution should land.
This is precisely where retention fails: strong erasure can damage semantically adjacent benign concepts~\citep{amara2025erasebench}, while methods that maximize a scalar forgetting reward can reduce benign diversity by concentrating generation onto a narrow set of concept-free outputs~\citep{uehara2024finetuning}.
Anchor-based methods avoid some direct-suppression damage by mapping the forget concept toward a designated target concept, but this introduces a separate choice of where the erased concept should go.
Fixed anchors can be unstable or semantically biased, and recent continual-unlearning evidence shows that such coarse mappings compound into severe retention collapse under sequential deletion requests~\citep{george2026locality}; adaptive anchors make the target itself context-dependent~\citep{bui2025fantastic}.
Moreover, such objectives do not by themselves specify the desired post-unlearning distribution, leaving open how probability mass should be redistributed after erasure.
These limitations motivate an anchor-free approach that removes mass from concept-expressing regions without selecting a replacement concept.
Figure~\ref{fig:main_methods_characteristics} illustrates this distinction by contrasting incomplete erasure, including collateral damage from direct suppression and diversity collapse under mode-seeking maximization, with the proposed minimum-deviation distributional target.

\begin{figure}[t]
\centering
\includegraphics[width=0.98\textwidth]{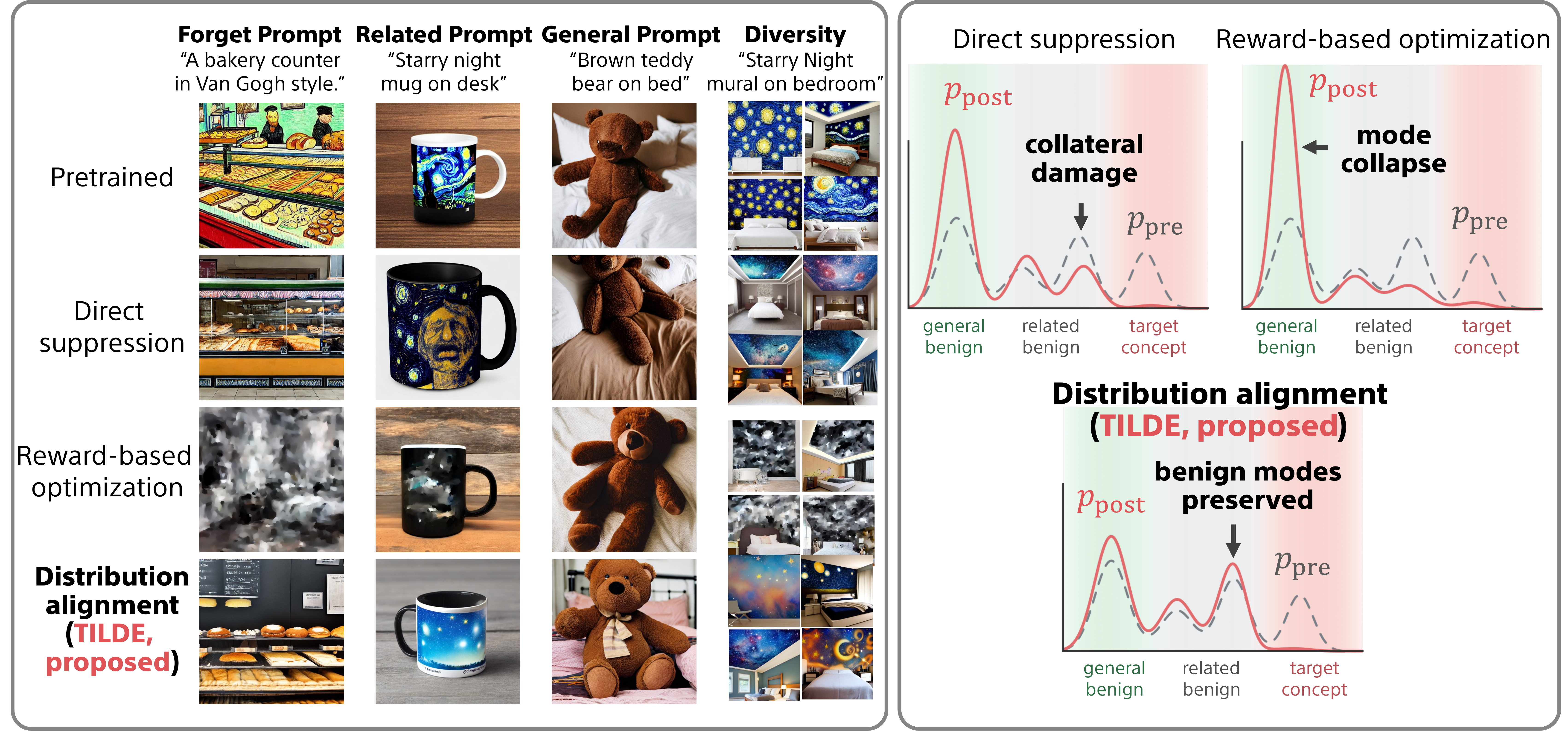}
\caption{Concept unlearning failure modes and qualitative examples for Van Gogh style removal.
Left: generations for forget, related, general, and diversity prompts under the pretrained model, direct suppression, reward-based optimization, and distribution alignment.
Right: schematic post-unlearning behavior over output regions.
Direct suppression can damage related benign modes, while reward-based optimization can collapse diversity onto a few safe modes.
Distribution alignment suppresses the target concept while preserving benign modes.}
\label{fig:main_methods_characteristics}
\end{figure}

We therefore argue that concept unlearning should be defined by its post-unlearning distribution before choosing an update rule.
The retain-only gold standard and recent evidence that reference-specific unlearning metrics can mask distributional misalignment between the unlearned model and the retain-only reference~\citep{cho2025fade}, suggest a minimal-deviation principle: among conditional distributions that sufficiently suppress the forget concept, choose the one closest to the original distribution induced by the pretrained model.
This turns unlearning into a constrained conditional distribution-optimization problem rather than a local suppression, anchor replacement, or reward maximization problem.
It also explains why retaining performance must be measured locally as well as globally: semantically nearby benign prompts are exactly where collateral damage is easiest to hide under coarse aggregate metrics.

We propose \textsc{TILDE} (TILt-based Distributional Erasure), which realizes this principle by tilting each prompt-conditioned pretrained distribution according to a forget energy.
Concept-expressing images are down-weighted, while benign images preserve their pretrained relative likelihoods up to normalization.
Section~\ref{sec:target} derives the corresponding constrained projection and its closed-form target. The resulting target is anchor-free and distributional: it specifies how probability mass should be reallocated after erasure, rather than choosing a replacement concept or maximizing a single safe output.
We therefore adopt GFlowNet-based training~\citep{bengio2023gflownet}, which targets sampling in proportion to an unnormalized terminal density.
To make this compatible with diffusion finetuning, we use residual $\nabla$-GFlowNet training~\citep{liu2025nablagfn}: the pretrained model supplies the base distribution, and the learner estimates only the score correction induced by the forget-energy tilt.

This perspective also guides evaluation: forgetting must be reported together with neighboring benign preservation and distributional fidelity to pretrained or retrained behavior.

\paragraph{Contributions.}
\begin{enumerate}
    \item We formulate concept unlearning as constrained distributional alignment, making the post-unlearning distribution target explicit and deriving the corresponding anchor-free, prompt-conditioned minimal-deviation Gibbs form.
    \item We introduce \textsc{TILDE}, which instantiates the energy-tilted target with residual $\nabla$-GFlowNet training in diffusion latent space, using a thresholded terminal forget energy, residual score matching, and LoRA updates.
    \item We evaluate unlearning of object, artistic-style and characters across forgetting, related/general retention, and distributional metrics including FID and FADE~\citep{cho2025fade}, showing improvements over prior baselines.
\end{enumerate}

\section{Target Distribution for Concept Unlearning}
\label{sec:target}

\subsection{Prompt-Conditioned Target via Constrained KL Projection}

We first make the desired post-unlearning conditional distribution explicit, instead of leaving it implicit in the choice of update rule.
Let $\pi(y)$ denote the fixed prompt population, including prompts that may elicit the concept $\mathcal{C}$, related benign prompts, and general benign prompts.
Let $E_{\mathcal C}(\bx)\ge 0$ be a measurable forget energy that quantifies concept evidence: larger values indicate that image $\bx$ more strongly expresses $\mathcal{C}$, whereas benign images, including semantically adjacent benign images, should have small or zero energy.
We define the KL-projection target by
\begin{align}
\label{eq:constrained_opt}
p^*_{\mathcal{C}}
&=
\argmin_{p}
\E_{y\sim\pi}\!\left[
\KL\!\left(p(\cdot \mid y) \,\|\, p_{\mathrm{pre}}(\cdot \mid y)\right)
\right]
\quad
\text{s.t.}
\quad
\E_{y\sim\pi,\,\bx\sim p(\cdot \mid y)}
\!\left[E_{\mathcal{C}}(\bx)\right]
\le \delta_{\mathcal{C}},
\end{align}
where the optimization is over conditional generators $p(\cdot \mid y)$, the prompt marginal $\pi(y)$ is held fixed, the KL term penalizes unnecessary deviation from the pretrained conditional, and the constraint bounds the expected concept evidence under the post-unlearning distribution to at most $\delta_{\mathcal{C}}>0$. Under standard regularity conditions for the constrained KL projection, this projection has the unique solution
\begin{align}
\label{eq:gibbs_target}
p^*_{\mathcal{C}}(\bx \mid y)
&=
\frac{p_{\mathrm{pre}}(\bx \mid y)\exp\!\big(-\beta E_{\mathcal{C}}(\bx)\big)}
{Z_{\beta,\mathcal{C}}(y)},
\end{align}
where $Z_{\beta,\mathcal{C}}(y)$ is the prompt-specific normalizing constant and $\beta \ge 0$ is the Lagrange multiplier associated with the forgetting constraint.
The tilt penalizes only concept-expressing images; benign images share the freed mass proportionally, so their pretrained relative likelihoods are preserved.

The remaining design questions are how to instantiate $E_{\mathcal{C}}$ and how to realize the resulting tilted distribution in a diffusion model.

\subsection{Practical Surrogate Target: CLIP-Based Thresholded Forget Energy}
\label{sec:reward}

In practice, we instantiate the abstract forget energy with a thresholded CLIP-based surrogate~\citep{radford2021clip}.
The CLIP score serves as a measure of concept evidence, while the threshold creates a no-penalty region for benign images.

Let $\mathcal{C}$ denote the semantic concept to be forgotten, and let $y$ denote a prompt that conditions image generation.
We use a small set of CLIP text descriptors $\mathcal{Q}_{\mathcal{C}}=\{q_1,\ldots,q_M\}$ as concept-evidence probes for $\mathcal{C}$, distinct from the generation prompt $y$. Each $q_i$ is a short, paraphrased reference to $\mathcal{C}$ that fixes the visual identity without describing scene context, so probing is robust to lexical variation. For Pikachu, for instance, $\mathcal{Q}_{\mathcal{C}}=\{$``Pikachu'', ``a photo of Pikachu'', ``an image of Pikachu the Pokemon'', ``Pikachu, the yellow Pokemon character''$\}$; analogous sets are used for artistic styles, e.g., ``Van Gogh style''.
For a generated image $\bx$, we compute the average CLIP similarity to these descriptors, $s_{\mathcal{C}}(\bx)=|\mathcal{Q}_{\mathcal{C}}|^{-1}\sum_{q\in\mathcal{Q}_{\mathcal{C}}}
\cos(\mathrm{CLIP}_{\mathrm{img}}(\bx),\,\mathrm{CLIP}_{\mathrm{txt}}(q))$, and refer to this quantity as CLIP concept evidence.
We then define the forget energy as
\begin{align}
\label{eq:reward}
\widetilde{E}_{\mathcal{C}}(\bx)
&=
\begin{cases}
0, & \text{if } s_{\mathcal{C}}(\bx)\le \tau, \\
\lambda_{\mathrm{scale}}
\left(
\exp\bigl(\alpha (s_{\mathcal{C}}(\bx)-\tau)\bigr)-1
\right), & \text{if } s_{\mathcal{C}}(\bx)>\tau.
\end{cases}
\end{align}
Here $\tau$ is the CLIP-score threshold, $\alpha$ controls the sharpness above the threshold, and $\lambda_{\mathrm{scale}}$ sets the overall penalty strength, including the Gibbs multiplier $\beta$.

The threshold is the key design choice.
Without it, a continuous penalty based on $s_{\mathcal C}(\bx)$ would exert nonzero gradients even on benign images with weak concept similarity.
The thresholded energy instead creates a no-penalty region below $\tau$, while applying the energy tilt primarily to high-evidence images above $\tau$.

This yields the practical surrogate target
\begin{align}
\label{eq:practical_target}
p_{\mathrm{tar},\mathcal{C}}(\bx \mid y)
&\propto
p_{\mathrm{pre}}(\bx \mid y)\exp\!\big(-\widetilde{E}_{\mathcal{C}}(\bx)\big).
\end{align}
This gives an anchor-free target: rather than prescribing a replacement concept, the tilt reweights the pretrained conditional density according to concept evidence.

\subsection{Implication for Optimization}

The target to realize is the prompt-conditioned tilted density in Eq.~\eqref{eq:practical_target}, rather than to seek modes that maximize a scalar forgetting reward.
Accordingly, the learning objective is to realize proportional sampling from this target, preserving benign-region coverage while down-weighting high-energy images.
We adopt residual $\nabla$-GFlowNet because diffusion models are score-parameterized and Eq.~\eqref{eq:practical_target} has a pretrained-times-energy-tilt structure that residual training matches directly.
The next section details this realization.

\section{Diffusion Realization via Residual \texorpdfstring{$\nabla$}{nabla}-GFlowNet}
\label{sec:optimization}

We fix $\mathcal{C}$ throughout this section and use the practical thresholded energy $\widetilde{E}_{\mathcal{C}}$ from Section~\ref{sec:reward}. Our diffusion realization instantiates the residual $\nabla$-GFlowNet machinery~\citep{liu2025nablagfn} with the forget-energy-tilted target in Eq.~\eqref{eq:practical_target}.
Consider a latent diffusion model~\citep{rombach2022high}.
For a prompt $y$, let $\bz_T\sim\mathcal{N}(\mathbf{0},\mathbf{I})$ be the initial noise and let $\bz_T\to\bz_{T-1}\to\cdots\to\bz_0$ denote the denoising chain, with decoded image $\bx=\mathcal{D}(\bz_0)$.
We write $P_F^\theta(\bz_{t-1}\mid\bz_t,y)$ for the LoRA-adapted denoising transition and $P_F^{\rm pre}(\bz_{t-1}\mid\bz_t,y)$ for the pretrained transition.
The subscript $F$ follows the GFlowNet convention~\citep{bengio2023gflownet} and denotes the sampling, or denoising, direction, not the forward noising process used in diffusion training.
We use standard DDPM notation, with $\alpha_t=1-\beta_t$ and $\bar{\alpha}_t=\prod_{s=1}^t\alpha_s$.

\begin{figure}[t]
	\centering
	\includegraphics[width=0.95\textwidth]{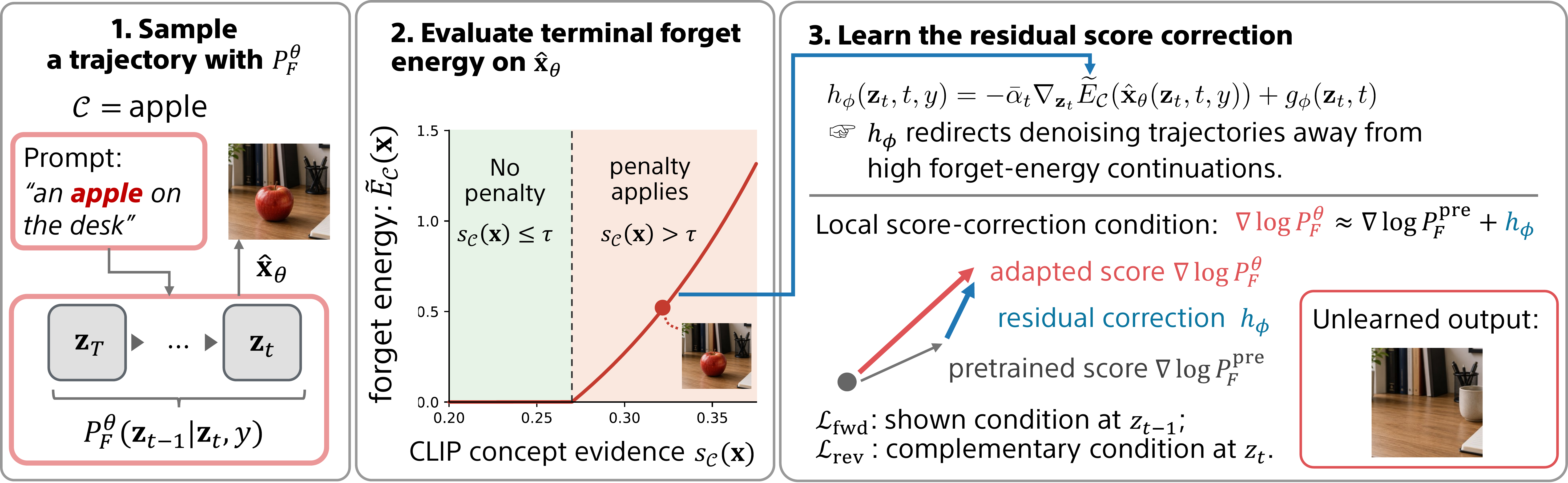}
    \caption{Overview of \textsc{TILDE} training in diffusion latent space. Given a target concept $\mathcal{C}$ and a sampled prompt $y$, the LoRA-adapted denoiser samples a trajectory $\bz_T \rightarrow \cdots \rightarrow \bz_0$, and evaluates the thresholded forget energy $\widetilde{E}_{\mathcal{C}}(\hat{\bx})$ from CLIP concept evidence. Residual $\nabla$-GFlowNet training then matches the score correction between $P_F^\theta(\bz_{t-1}\mid\bz_t,y)$ and $P_F^{\rm pre}(\bz_{t-1}\mid\bz_t,y)$ using the continuation score $h_\phi(\bz_t,t,y)$, steering denoising trajectories away from high-energy concept-expressing outputs while preserving low-energy benign dynamics.}
	\label{fig:method_overview}
\end{figure}

\paragraph{Residual balance condition.}
Eq.~\eqref{eq:practical_target} defines a clean-image tilt of the pretrained sampler.
The remaining question is how this terminal reweighting should be propagated through the denoising step.
We introduce a positive residual continuation weight $\widetilde{F}_t(\bz_t,y)$.
Intuitively, it is the benign mass still reachable from $\bz_t$: large when continuations tend to end in low-energy images, and small when they tend to end in concept-expressing high-energy images.
The adapted transition should therefore change the pretrained transition only by favoring next states with larger continuation weight:
\begin{align}
\label{eq:residual_ratio_balance}
\frac{P_F^\theta(\bz_{t-1}\mid\bz_t,y)}
{P_F^{\rm pre}(\bz_{t-1}\mid\bz_t,y)}
&=
\frac{\widetilde{F}_{t-1}(\bz_{t-1},y)}
{\widetilde{F}_t(\bz_t,y)} .
\end{align}
This condition is useful because it factorizes the target path-probability ratio locally.
Along any trajectory $\bz_T \to \cdots \to \bz_0$, the intermediate $\widetilde{F}_t$ terms cancel, leaving only $\widetilde{F}_0(\bz_0,y)/\widetilde{F}_T(\bz_T,y)$.
With $\widetilde{F}_0(\bz_0,y)\propto\exp(-\widetilde{E}_{\mathcal{C}}(\mathcal{D}(\bz_0)))$ and $\widetilde{F}_T(\bz_T,y)=c_{\mathrm{init}}(y)$ independent of $\bz_T$, the terminal marginal is proportional to $p_{\rm pre}(\bz_0\mid y)\exp(-\widetilde{E}_{\mathcal{C}}(\mathcal{D}(\bz_0)))$, matching Eq.~\eqref{eq:practical_target}.
Thus the condition converts a terminal energy into proportional sampling: concept-expressing images lose mass, while benign modes keep their pretrained relative proportions.

Since $\widetilde{F}_t$ is unknown, we learn only its score
\begin{align*}
h_\phi(\bz_t,t,y) := \nabla_{\bz_t}\log \widetilde{F}_t(\bz_t,y),
\end{align*}
and use it to define the score-form training objective below.
Training jointly optimizes the LoRA parameters and the auxiliary continuation-score parameters $\phi$.

\paragraph{Score-form training objective.}
Directly enforcing Eq.~\eqref{eq:residual_ratio_balance} is inconvenient in diffusion models because it involves transition densities.
Following residual $\nabla$-GFlowNet~\citep{liu2025nablagfn}, we therefore use its score-form counterpart, which can be expressed using the local score information provided by the denoiser.
Differentiating Eq.~\eqref{eq:residual_ratio_balance} with respect to $\bz_{t-1}$ and using that $\widetilde{F}_t(\bz_t,y)$ does not depend on $\bz_{t-1}$ gives the local score-correction condition:
\begin{align}
\label{eq:residual_forward_db}
\nabla_{\bz_{t-1}}\log P_F^\theta(\bz_{t-1}\mid\bz_t,y)
-
\nabla_{\bz_{t-1}}\log P_F^{\rm pre}(\bz_{t-1}\mid\bz_t,y)
&=
h_\phi(\bz_{t-1},t-1,y).
\end{align}
The left-hand side is the score correction introduced by finetuning, measuring the change in local probability flow relative to the pretrained denoiser.
The right-hand side is the auxiliary estimate of the residual continuation score required by the same balance condition.
We enforce this local consistency by penalizing the mismatch between the two vector fields, yielding the forward loss
\begin{align}
\label{eq:loss_forward_main}
\mathcal{L}_{\mathrm{fwd}}
&=
\E\!\left[
\left\|
\nabla_{\bz_{t-1}}\log P_F^\theta(\bz_{t-1}\mid\bz_t,y)
-
\nabla_{\bz_{t-1}}\log P_F^{\rm pre}(\bz_{t-1}\mid\bz_t,y)
-
h_\phi(\bz_{t-1},t-1,y)
\right\|^2
\right].
\end{align}
We also impose the complementary score condition obtained by differentiating the same residual balance relation with respect to $\bz_t$, yielding the reverse loss $\mathcal{L}_{\mathrm{rev}}$.
Together, $\mathcal{L}_{\mathrm{fwd}}$ and $\mathcal{L}_{\mathrm{rev}}$ make the residual correction locally consistent across each denoising transition.
Unlike standard denoising-score training~\citep{ho2020denoising,song2021scorebased}, the target is not ground-truth noise or a data score, but the residual correction that turns the pretrained sampler into the energy-tilted unlearned sampler.

The residual score $h_\phi(\bz_t,t,y)$ is the score of the residual continuation weight induced by the terminal forget energy.
Since this score is not available in closed form at intermediate noisy states, we use the forward-looking residual parameterization
\begin{align}
\label{eq:forward_looking}
h_\phi(\bz_t,t,y)
&=
-\bar{\alpha}_t
\nabla_{\bz_t}\widetilde{E}_{\mathcal{C}}\!\left(\hat{\bx}_\theta(\bz_t,t,y)\right)
+
g_\phi(\bz_t,t).
\end{align}
Here $\hat{\bx}_\theta(\bz_t,t,y)$ is the one-step clean-image estimate predicted from the current noisy latent, and $\bar{\alpha}_t$ is the standard cumulative DDPM signal coefficient.
The first term is the one-step pullback of the terminal-energy score, and $g_\phi$ parameterizes the remaining approximation residual of this continuation score.

\paragraph{Interpretation of $h_\phi$.}
The residual field $h_\phi$ can be interpreted as a state-dependent guide for how the adapted denoiser should deviate from the pretrained denoiser. Instead of manually specifying separate forget and retain directions, or selecting an anchor concept, TILDE lets the thresholded forget energy determine this direction from the current one-step clean estimate $\hat{x}_\theta(z_t,t,y)$. When the estimate has high concept evidence, $\nabla_{z_t}\tilde{E}_C(\hat{x}_\theta)$ is nonzero, so $h_\phi$ induces an active residual score correction that redirects the denoising trajectory away from concept-expressing regions. When the estimate lies below the threshold, the energy-pullback term is exactly zero, so the residual balance condition demands that the finetuned transition match the pretrained transition at these states. Benign trajectories are therefore encouraged to remain close to pretrained behavior.

\paragraph{Overall loss and design rationale.}
We jointly optimize the LoRA update $\Delta\theta$ and the auxiliary continuation-score parameters $\phi$ by minimizing
\begin{align}
\label{eq:loss_total}
\mathcal{L}
&=
\mathcal{L}_{\mathrm{fwd}}
+
\lambda_{\mathrm{rev}}\mathcal{L}_{\mathrm{rev}}
+
\lambda_{\mathrm{term}}\mathcal{L}_{\mathrm{term}}.
\end{align}
Here, $\lambda_{\mathrm{rev}}$ and $\lambda_{\mathrm{term}}$ are balancing weights, and $\mathcal{L}_{\mathrm{term}}$ penalizes $\|g_\phi(\bz_0,0)\|_2^2$, keeping the clean-end boundary condition determined solely by the forget energy.

Our contribution is the unlearning-specific distributional target and its full realization in diffusion latent space; residual $\nabla$-GFlowNet provides the optimization backbone that makes this realization practical.
Each component follows from this target.
The residual form is appropriate because Eq.~\eqref{eq:practical_target} changes each prompt-conditioned pretrained distribution by a multiplicative energy tilt.
The score form is appropriate because diffusion models expose denoising scores.
The forward-looking correction is appropriate because the forget signal is terminal, while the denoising chain is long.
Finally, the thresholded energy from Section~\ref{sec:reward} yields no direct terminal-energy force outside the forget region, which is essential for preserving nearby benign concepts.
We also restrict updates to LoRA adapters~\citep{hu2022lora}, adding a parameter-space minimal-update bias to the distribution-space minimal-deviation target.
    
\section{Related Work}
\label{sec:related}

\paragraph{Current concept-unlearning methods.}
Existing approaches include direct suppression and selective finetuning (e.g., ESD, SalUn, and Selective Amnesia)~\citep{gandikota2023erasing,fan2024salun,heng2023selective}, target-space editing frameworks such as UCE and MACE~\citep{gandikota2024unified,lu2024mace}, influence and trajectory-steering methods~\citep{wu2024scissorhands,li2025set,wu2025erasing}, and preference/reward-based optimization such as DUO, DDPO, and DRaFT~\citep{park2024duo,black2024ddpo,clark2024draft}.
These methods can remove target concepts effectively, but they often leave the final post-unlearning distribution implicit.

\paragraph{Anchors, mapping, and dynamic targets.}
Many concept-erasure methods can be interpreted as choosing a destination representation and mapping the forget concept toward that target.
Concept Ablation uses explicit anchor replacement~\citep{kumari2023ablating}, while other methods use broader target mappings at the embedding or attention level~\citep{gandikota2024unified,lu2024mace}.
AGE argues that fixed generic anchors are often suboptimal and proposes adaptive/dynamic target selection based on concept context~\citep{bui2025fantastic}.
LACU further shows that fixed anchors are especially brittle when deletion requests arrive sequentially, because large per-step displacements accumulate and damage neighboring concepts~\citep{george2026locality}.

\paragraph{GFlowNet-based and distributional methods.}
Our work builds on $\nabla$-GFlowNet~\citep{liu2025nablagfn}, originally developed for reward-aligned sampling in diffusion models.
EraseFlow~\citep{kusumba2025eraseflow} is the closest unlearning-specific GFlowNet baseline, which uses a trajectory-balance (TB) loss~\citep{malkin2022trajectory} on approximated log-probabilities; compared with it, we use score-based residual $\nabla$-DB and learn only the deviation from the pretrained model rather than training a full sampler from scratch.

\paragraph{Failures, robustness, and side effects.}
Recent evidence shows that apparent forgetting is often fragile: erased concepts can be revived or bypassed under adversarial prompting and model probing~\citep{hsu2024ring,zhang2024generate}, and relearning/revival remains a practical threat~\citep{george2025illusion,gao2025meta,lu2025concepts}.
This has motivated robust unlearning methods that incorporate adversarial training and robustness-aware objectives~\citep{zhang2024defensive,srivatsan2025stereo}.
Another persistent failure mode is the ripple effect on nearby concepts, where semantically adjacent benign concepts are unintentionally degraded after unlearning~\citep{amara2025erasebench}.
Related retention-focused unlearning work studies gradient conflicts between forget and retain objectives as one mechanism behind this trade-off~\citep{patel2025learning}.
Complementary analysis shows that gradient-ascent unlearning can fail because forget and retain data are statistically coupled, so ascent can move the model away from the retrained reference rather than toward it~\citep{mavrothalassitis2026ascent}.

\section{Experiments}
\label{sec:experiments}

\subsection{Experimental Setup}
\label{sec:setup}

\noindent\textbf{Base Model and Training.}
All experiments are performed on Stable Diffusion v1.5~\citep{rombach2022high}, the standard base model for concept unlearning research.
We finetune LoRA adapters (rank $r = 8$) on the attention layers using AdamW with DDPM sampling; full training details and results on other SD variants are provided in the supplementary material.

\noindent\textbf{Concepts.}
We evaluate across four categories: objects, characters (e.g., celebrities, fictional characters), artistic styles, and nudity, covering a total of approximately \textbf{60 distinct concepts} including targets, semantically related concepts, and general themes.

\noindent\textbf{Unlearning Prompts.}
For each concept $\mathcal{C}$, we construct a forget prompt set $\mathcal{Y}_{\mathcal{C}}$ with \textbf{30 diverse prompts} using GPT-4o-mini, guided by a system prompt that enforces visual grounding (the concept must be a clearly depictable, in-frame subject), scene and category diversity (varied indoor/outdoor settings, framings, and activities), and lexical simplicity (short, single-clause natural-language phrasings).
The same prompt generation pipeline is reused, with minor wording changes to the system prompt, when retain prompts are added to training (Section~\ref{sec:results}, Q5): we first ask the LLM to enumerate related concepts and then sample concept-disjoint scenes for each.
Unless stated otherwise, $\mathcal{Q}_{\mathcal{C}}$ contains the canonical text descriptor of $\mathcal{C}$.

\subsection{Evaluation Metrics}
\label{sec:metrics}

For every concept, we use \textbf{20 unique evaluation prompts} and generate \textbf{8 images per prompt}, giving approximately \textbf{10,000 images} per model checkpoint.
All evaluation prompts are unseen during training.

\noindent\textbf{VLM-Based Accuracy Metrics.}
We use Qwen2.5-VL-7B-Instruct~\citep{bai2025qwen2} as our automated evaluator, querying it with binary questions on generated images.
VLMs are substantially more reliable than ImageNet-style classifiers or UnlearnCanvas~\citep{zhang2024unlearncanvas} style-classifier heads, which struggle with diverse concepts such as celebrities and artistic styles; supporting evidence is in the supplementary.

\begin{itemize}[leftmargin=1.5em, itemsep=1pt, topsep=2pt]
  \item \textbf{Unlearning Accuracy ($U_{\mathrm{acc}}$ $\uparrow$) \& CLIP Score ($U_{\mathrm{clip}}$ $\uparrow$):}
    $U_{\mathrm{acc}}$ measures forgetting effectiveness; $U_{\mathrm{clip}}$ measures ``in-prompt retainability''~\citep{ren2025sixcd}, ensuring benign parts of the prompt are still generated correctly.
  \item \textbf{Related Retention ($RR_{\mathrm{acc}}$ $\uparrow$; $RR_{\mathrm{clip}}$ $\uparrow$):}
    $RR_{\mathrm{acc}}$ measures unintended collateral damage on semantically adjacent concepts (e.g., testing ``Impressionism'' after unlearning ``Van Gogh''); $RR_{\mathrm{clip}}$ verifies alignment is preserved.
  \item \textbf{General Retention ($GR_{\mathrm{acc}}$ $\uparrow$; $GR_{\mathrm{clip}}$ $\uparrow$):}
    $GR_{\mathrm{acc}}$ assesses overall knowledge preservation on general, unrelated prompts; $GR_{\mathrm{clip}}$ verifies text-to-image alignment remains intact.
\end{itemize}

\noindent\textbf{Distributional Metrics: FID and FADE.}
We follow the UnlearnCanvas~\citep{zhang2024unlearncanvas} evaluation protocol for \textbf{FID}~\citep{heusel2017fid}, measuring distributional quality of generated images on retain prompts.

For a more principled distributional assessment, we also report \textbf{FADE}~\citep{cho2025fade}, which measures functional alignment between the unlearned model and a gold-standard retain-only model by comparing bidirectional likelihood assignments over generated samples:
\begin{align}
\label{eq:fade}
\mathrm{FADE} \coloneqq
\mathbb{E}_{\bx \sim p_{\mathrm{retain}}(\cdot\mid y)}\!\left[\log\frac{p_{\mathrm{retain}}(\bx\mid y)}{p_{\mathrm{unlearn}}(\bx\mid y)}\right]
+
\mathbb{E}_{\bx \sim p_{\mathrm{unlearn}}(\cdot\mid y)}\!\left[\log\frac{p_{\mathrm{unlearn}}(\bx\mid y)}{p_{\mathrm{retain}}(\bx\mid y)}\right].
\end{align}
FADE is non-negative and equals zero only when the two distributions are identical; a lower score indicates closer alignment to genuine unlearning.
For diffusion models, FADE is tractably estimated via weighted differences of denoising MSE losses~\citep{cho2025fade}.

\subsection{Experimental Results and Analysis}
\label{sec:results}

\begin{table}[t]
\centering
\caption{Category-level concept unlearning performance averaged over concepts in each group. Base Stable Diffusion v1.5 reported \textbf{SD v1.5 base: Accuracy 89\%, CLIP score 32.6}.
\textbf{Objects} averages apple, banana, golf ball, and cat.
\textbf{Characters} averages Brad Pitt, Lionel Messi, Mickey Mouse, and Pikachu.
\textbf{Style} averages cartoon style, Van~Gogh style, and Monet style.
Each category reports unlearning accuracy ($U_{\mathrm{Acc}}$), related-retain accuracy ($RR_{\mathrm{Acc}}$), and general-retain accuracy ($GR_{\mathrm{Acc}}$).
The final columns report average FADE and FID over the style concepts when available.}
\label{tab:main_comparison}
\vspace{0.5em}
\setlength{\tabcolsep}{2.7pt}
\resizebox{\textwidth}{!}{%
\begin{tabular}{@{}lccccccccccc@{}}
\toprule
\multirow{2}{*}{\textbf{Method}}
& \multicolumn{3}{c}{\textbf{Objects}}
& \multicolumn{3}{c}{\textbf{Characters}}
& \multicolumn{3}{c}{\textbf{Style}}
& \multirow{2}{*}{\shortstack{\textbf{Avg.} \\ \textbf{FADE} $\downarrow$}}
& \multirow{2}{*}{\shortstack{\textbf{Avg.} \\ \textbf{FID} $\downarrow$}} \\
\cmidrule(lr){2-4}\cmidrule(lr){5-7}\cmidrule(l){8-10}
& $U_{\mathrm{Acc}} \uparrow$ & $RR_{\mathrm{Acc}} \uparrow$ & $GR_{\mathrm{Acc}} \uparrow$
& $U_{\mathrm{Acc}} \uparrow$ & $RR_{\mathrm{Acc}} \uparrow$ & $GR_{\mathrm{Acc}} \uparrow$
& $U_{\mathrm{Acc}} \uparrow$ & $RR_{\mathrm{Acc}} \uparrow$ & $GR_{\mathrm{Acc}} \uparrow$
& & \\
\midrule
ESD-u & 0.74 & 0.58 & 0.79 & \cellcolor{tabthird}0.97 & 0.36 & 0.77 & 0.55 & \cellcolor{tabthird}0.81 & \cellcolor{tabthird}0.84 & 254 & 10.0 \\
ESD-x & 0.77 & 0.58 & 0.78 & \cellcolor{tabsecond}0.98 & 0.23 & 0.55 & 0.70 & 0.71 & 0.82 & 200 & 9.5 \\
UCE & 0.62 & \cellcolor{tabsecond}0.87 & 0.83 & 0.92 & 0.57 & 0.67 & 0.29 & \cellcolor{tabfirst}0.93 & \cellcolor{tabsecond}0.87 & \cellcolor{tabsecond}138 & 14.8 \\
CA & 0.61 & \cellcolor{tabfirst}0.91 & \cellcolor{tabfirst}0.90 & \cellcolor{tabthird}0.97 & \cellcolor{tabthird}0.59 & \cellcolor{tabsecond}0.87 & \cellcolor{tabthird}0.86 & 0.72 & \cellcolor{tabsecond}0.87 & 152 & \cellcolor{tabsecond}7.3 \\
MACE & 0.44 & 0.79 & 0.83 & \cellcolor{tabfirst}0.99 & 0.17 & 0.22 & \cellcolor{tabfirst}1.00 & 0.03 & 0.01 & 67370 & 181.1 \\
DUO & \cellcolor{tabthird}0.91 & 0.76 & \cellcolor{tabsecond}0.88 & \cellcolor{tabthird}0.97 & 0.46 & \cellcolor{tabthird}0.86 & 0.56 & 0.49 & \cellcolor{tabsecond}0.87 & 358 & \cellcolor{tabthird}9.0 \\
EraseFlow & 0.82 & 0.53 & 0.72 & \cellcolor{tabfirst}0.99 & 0.16 & 0.57 & \cellcolor{tabthird}0.86 & 0.52 & 0.72 & 400 & 17.6 \\
Meta & \cellcolor{tabsecond}0.95 & 0.36 & 0.72 & \cellcolor{tabfirst}0.99 & 0.21 & 0.71 & 0.69 & 0.27 & 0.56 & 40435 & 21.9 \\
SHS~\citep{wu2024scissorhands} & 0.84 & 0.30 & 0.35 & 0.85 & 0.34 & 0.32 & 0.80 & 0.17 & 0.30 & \cellcolor{tabthird}150 & 171.3 \\
EDiff~\citep{wu2025erasing} & 0.39 & \cellcolor{tabthird}0.83 & \cellcolor{tabsecond}0.88 & 0.52 & \cellcolor{tabfirst}0.78 & \cellcolor{tabfirst}0.88 & 0.32 & \cellcolor{tabsecond}0.85 & \cellcolor{tabfirst}0.88 & 287 & \cellcolor{tabfirst}6.8 \\
\midrule
\textbf{\textsc{TILDE}} & \cellcolor{tabfirst}0.97 & 0.72 & \cellcolor{tabthird}0.86 & \cellcolor{tabfirst}0.99 & \cellcolor{tabsecond}0.64 & \cellcolor{tabsecond}0.87 & \cellcolor{tabsecond}0.88 & 0.76 & \cellcolor{tabthird}0.84 & \cellcolor{tabfirst}130 & 9.6 \\
\bottomrule
\end{tabular}%
}
\end{table}

\raggedbottom

\paragraph{Q1: How well do existing methods retain overall model performance during unlearning?}

Table~\ref{tab:main_comparison} reveals a consistent failure mode: strong forgetting comes at the expense of collateral damage.
Hard erasure (MACE) reaches near-perfect unlearning but destroys related and general-retain accuracy; softer suppression (ESD variants) preserves benign behavior but rarely reaches the needed level of erasure.
Anchor-based CA~\citep{kumari2023ablating} and UCE~\citep{gandikota2024unified} rely on \emph{explicit retain supervision} (anchor-concept distillation and a preserve-set attention edit, respectively), which buys image quality but still leaves residual concept evidence.
Preference-based DUO is inconsistent across concept types, and the closest GFlowNet baseline (EraseFlow) shows particularly sharp retention drops on adjacent concepts.

\paragraph{Q2: How well does \textsc{TILDE} perform?}

\textsc{TILDE} achieves near-complete forgetting across all category types while substantially improving retention over methods at the same unlearning operating point — and it does so \emph{without} the explicit retain set or anchor concept that UCE and CA depend on, placing it on the Pareto frontier.
With retain prompts optionally added, retention improves further (Q5).
The advantage is most pronounced on character and style concepts; on objects, related-retain remains harder due to tighter semantic coupling with neighboring concepts.

\paragraph{Q3: Why does the GFlowNet formulation help?}

We ablate each ingredient of the residual $\nabla$-GFlowNet objective on the Pikachu concept (Table~\ref{tab:gfn_ablation}); recall that \citep{liu2025nablagfn} replaces intractable absolute log-flow targets with a gradient-form (score-matching) loss well-conditioned in latent space, and that the residual $g_\phi$ restricts learning to the score \emph{correction} relative to the pretrained model.

\noindent\begin{minipage}[t]{0.48\textwidth}
\vspace{0pt}
The $\nabla$ form is the primary driver of forgetting: dropping to 0th-order DB collapses unlearning to 0.40, since absolute-flow targets are too poorly conditioned in latent space to suppress the concept.
The residual $g_\phi$ is the primary driver of retention: removing it forces the full score field to absorb the forget update, dropping general retention to 0.83.
DDPO-style reward maximization keeps forgetting strong but collapses general retention to 0.33, confirming that GFlowNet's proportional-sampling property is what prevents mode collapse onto concept-free outputs.
\end{minipage}%
\hfill
\begin{minipage}[t]{0.48\textwidth}
\vspace{0pt}
\centering
\captionof{table}{Ablation of the $\nabla$-GFlowNet formulation on Pikachu concept unlearning.
All variants share the same thresholded forget energy and LoRA parameterization.}
\label{tab:gfn_ablation}
\vspace{0.5em}
\setlength{\tabcolsep}{4pt}
\resizebox{\linewidth}{!}{%
\begin{tabular}{@{}lcccc@{}}
\toprule
\textbf{Method} & $U_{\mathrm{acc}} \uparrow$ & $RR_{\mathrm{acc}} \uparrow$ & $GR_{\mathrm{acc}} \uparrow$ \\
\midrule
\textbf{\textsc{TILDE}} (resi $\nabla$-GFlowNet) & 0.99 & 0.76 & 0.88 \\
w/o residual                               & 0.96 & 0.79 & 0.87 \\
w/o $\nabla$ form                          & 0.40 & 0.81 & 0.87 \\
w/o $g_\phi$                               & 0.97 & 0.72 & 0.83 \\
DDPO (reward max.)                         & 0.98 & 0.61 & 0.33 \\
\bottomrule
\end{tabular}%
}
\end{minipage}

\paragraph{Q4: How does the energy threshold $\tau$ affect unlearning?}

The thresholded forget energy (Eq.~\ref{eq:reward}) is a key design choice: only images with CLIP concept evidence above $\tau$ receive a direct terminal-energy penalty.
Figure~\ref{fig:tau_ablation} traces the accuracy and retention trade-offs across a range of threshold values on the Pikachu concept.

\noindent\begin{minipage}[t]{0.48\textwidth}
\vspace{0pt}
Removing the threshold entirely (not shown) causes catastrophic retention collapse: $U_{\mathrm{acc}}{=}1.00$ but $RR_{\mathrm{acc}}{=}0.08$, $GR_{\mathrm{acc}}{=}0.49$, since every sample receives an energy penalty regardless of concept relevance.
As $\tau$ rises from a low value, forgetting stays strong while retention steadily improves, reflecting the growing no-tilt region that shields benign samples from direct terminal-energy gradients.
Our chosen $\tau{=}0.22$ strikes the point where forgetting remains near-complete and retention peaks; pushing $\tau$ higher causes forgetting to degrade sharply while retention gains diminish.
The reason: below $\tau$, $\widetilde{E}_{\mathcal{C}}(\bx) = 0$ exactly, so the explicit terminal-energy force vanishes for those clean predictions, leaving only indirect effects through residual-flow propagation and shared LoRA parameters.
\end{minipage}%
\hfill
\begin{minipage}[t]{0.48\textwidth}
\vspace{0pt}
\centering
\includegraphics[width=\linewidth]{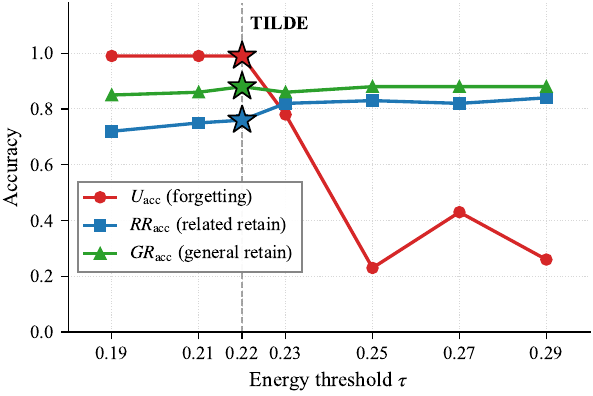}
\captionof{figure}{Effect of the energy threshold $\tau$ on Pikachu concept unlearning. Stars mark our chosen $\tau{=}0.22$ (\textsc{TILDE}). Removing the threshold entirely yields $(U,RR,GR)=(1.00, 0.08, 0.49)$ (off-axis).}
\label{fig:tau_ablation}
\end{minipage}

\paragraph{Q5: Does including retain prompts during unlearning help further?}

\begin{wraptable}{r}{0.48\textwidth}
\vspace{-1.5em}
\centering
\caption{Effect of including retain prompts during unlearning on Pikachu concept.
Adding retain prompts improves overall retention without affecting unlearning accuracy.}
\label{tab:retain_prompts}
\vspace{0.3em}
\setlength{\tabcolsep}{4pt}
\resizebox{\linewidth}{!}{%
\begin{tabular}{@{}lccc@{}}
\toprule
\textbf{Setting} & $U_{\mathrm{acc}} \uparrow$ & $RR_{\mathrm{acc}} \uparrow$ & $GR_{\mathrm{acc}} \uparrow$ \\
\midrule
\textsc{TILDE} & 0.99 & 0.76 & 0.88 \\
\textsc{TILDE} + General retain & 0.98 & 0.82 & 0.88 \\
\textsc{TILDE} + Related retain & 0.99 & 0.84 & 0.88 \\
\bottomrule
\end{tabular}%
}
\vspace{-1em}
\end{wraptable}

Since benign images below $\tau$ receive no direct terminal-energy penalty, general or related-retain prompts can be naturally mixed into the training batch as a regularizer toward pretrained benign behavior (Table~\ref{tab:retain_prompts}).

Both variants improve related-retain accuracy without measurably affecting unlearning, consistent with the thresholded-energy prediction: below-threshold samples receive no concept-suppression gradient and act only as anchors toward pretrained behavior.
Related-retain prompts give a slightly larger gain than general ones, since they probe the exact neighborhood where collateral damage leaks in through shared LoRA parameters.
Crucially, this benefit is a free byproduct of thresholding: unlike CA and UCE, our pipeline absorbs retain prompts as ordinary training samples with no extra loss term, so the Pareto improvement of Q2 strictly extends rather than trades off when retain supervision is available.

\section{Discussion}
\label{sec:discussion}

This work casts concept unlearning as distributional alignment rather than local suppression or anchor replacement: \textsc{TILDE} derives a minimum-deviation, energy-tilted target under a forgetting constraint, instantiated with thresholded CLIP energy and residual $\nabla$-GFlowNet training in diffusion latent space.
The results support the main prediction: strong forgetting without broad collateral damage when the target distribution is specified explicitly.
More broadly, this framing unifies existing concept-erasure methods: hard-suppression objectives correspond to extreme tilts, anchor-based methods substitute a manually chosen target distribution, and reward-maximization variants drop the proportional-sampling property needed for diversity preservation.
Within this view, \textsc{TILDE} is the natural minimal modification of the pretrained model that satisfies the forgetting constraint, while residual $\nabla$-GFlowNet training makes the alignment efficient in practice.

\paragraph{Limitations and future work.}
The forget energy must reliably measure concept evidence; CLIP similarity suffices for the objects, characters, identities and styles we evaluate, but ambiguous or highly compositional concepts may need stronger energies, such as VLM ensembles or task-specific classifiers.
The same residual $\nabla$-GFlowNet machinery transfers without modification once a calibrated energy is provided.
The threshold $\tau$ is a hyperparameter, but in our experiments it remained stable across concepts within a category.

\bibliographystyle{plainnat}
\bibliography{references}

@inproceedings{zhang2024unlearncanvas,
  title={{UnlearnCanvas}: A Stylized Image Dataset for Enhanced Machine Unlearning Evaluation in Diffusion Models},
  author={Zhang, Yihua and Fan, Chongyu and Zhang, Yimeng and Yao, Yuguang and Jia, Jinghan and Liu, Jiancheng and Zhang, Gaoyuan and Liu, Gaowen and Kompella, Ramana and Liu, Xiaoming and Liu, Sijia},
  booktitle={Thirty-eighth Conference on Neural Information Processing Systems Datasets and Benchmarks Track},
  year={2024}
}

@inproceedings{george2025illusion,
  title={The Illusion of Unlearning: The Unstable Nature of Machine Unlearning in Text-to-Image Diffusion Models},
  author={George, Naveen and Dasaraju, Karthik Nandan and Chittepu, Rutheesh Reddy and Mopuri, Konda Reddy},
  booktitle={Proceedings of the Computer Vision and Pattern Recognition Conference},
  pages={13393--13402},
  year={2025}
}

@inproceedings{george2026locality,
  title={Locality-Aware Continual Unlearning for Diffusion Models},
  author={George, Naveen and Murata, Naoki and Takida, Yuhta and Mopuri, Konda Reddy and Mitsufuji, Yuki},
  booktitle={European Conference on Computer Vision},
  year={2026},
  organization={Springer}
}

@inproceedings{hsu2024ring,
  title={Ring-A-Bell! How Reliable are Concept Removal Methods for Diffusion Models?},
  author={Tsai, Yu-Lin and Hsu, Chia-Yi and Xie, Chulin and Lin, Chih-Hsun and Chen, Jia You and Li, Bo and Chen, Pin-Yu and Yu, Chia-Mu and Huang, Chun-Ying},
  booktitle={12th International Conference on Learning Representations},
  year={2024}
}

@inproceedings{zhang2024generate,
  title={To generate or not? safety-driven unlearned diffusion models are still easy to generate unsafe images... for now},
  author={Zhang, Yimeng and Jia, Jinghan and Chen, Xin and Chen, Aochuan and Zhang, Yihua and Liu, Jiancheng and Ding, Ke and Liu, Sijia},
  booktitle={European Conference on Computer Vision},
  pages={385--403},
  year={2024},
  organization={Springer}
}

@inproceedings{gandikota2024unified,
  title={Unified concept editing in diffusion models},
  author={Gandikota, Rohit and Orgad, Hadas and Belinkov, Yonatan and Materzy{\'n}ska, Joanna and Bau, David},
  booktitle={Proceedings of the IEEE/CVF Winter Conference on Applications of Computer Vision},
  pages={5111--5120},
  year={2024}
}

@inproceedings{gandikota2023erasing,
  title={Erasing concepts from diffusion models},
  author={Gandikota, Rohit and Materzynska, Joanna and Fiotto-Kaufman, Jaden and Bau, David},
  booktitle={Proceedings of the IEEE/CVF international conference on computer vision},
  pages={2426--2436},
  year={2023}
}

@article{heng2023selective,
  title={Selective amnesia: A continual learning approach to forgetting in deep generative models},
  author={Heng, Alvin and Soh, Harold},
  journal={Advances in Neural Information Processing Systems},
  volume={36},
  pages={17170--17194},
  year={2023}
}

@inproceedings{kumari2023ablating,
  title={Ablating concepts in text-to-image diffusion models},
  author={Kumari, Nupur and Zhang, Bingliang and Wang, Sheng-Yu and Shechtman, Eli and Zhang, Richard and Zhu, Jun-Yan},
  booktitle={Proceedings of the IEEE/CVF International Conference on Computer Vision},
  pages={22691--22702},
  year={2023}
}

@inproceedings{fan2024salun,
  title={SalUn: Empowering Machine Unlearning via Gradient-based Weight Saliency in Both Image Classification and Generation},
  author={Chongyu Fan and Jiancheng Liu and Yihua Zhang and Eric Wong and Dennis Wei and Sijia Liu},
  booktitle={The Twelfth International Conference on Learning Representations},
  year={2024},
  url={https://openreview.net/forum?id=gn0mIhQGNM}
}

@inproceedings{lu2024mace,
  title={Mace: Mass concept erasure in diffusion models},
  author={Lu, Shilin and Wang, Zilan and Li, Leyang and Liu, Yanzhu and Kong, Adams Wai-Kin},
  booktitle={Proceedings of the IEEE/CVF Conference on Computer Vision and Pattern Recognition},
  pages={6430--6440},
  year={2024}
}

@inproceedings{huang2024receler,
  title={Receler: Reliable concept erasing of text-to-image diffusion models via lightweight erasers},
  author={Huang, Chi-Pin and Chang, Kai-Po and Tsai, Chung-Ting and Lai, Yung-Hsuan and Yang, Fu-En and Wang, Yu-Chiang Frank},
  booktitle={European Conference on Computer Vision},
  pages={360--376},
  year={2024},
  organization={Springer}
}

@inproceedings{wu2024scissorhands,
  title={Scissorhands: Scrub data influence via connection sensitivity in networks},
  author={Wu, Jing and Harandi, Mehrtash},
  booktitle={European Conference on Computer Vision},
  pages={367--384},
  year={2024},
  organization={Springer}
}

@inproceedings{amara2025erasebench,
    title={Erasing More Than Intended? How Concept Erasure Degrades the Generation of Non-Target Concepts}, 
    author={Ibtihel Amara and Ahmed Imtiaz Humayun and Ivana Kajic and Zarana Parekh and Natalie Harris and Sarah Young and Chirag Nagpal and Najoung Kim and Junfeng He and Cristina Nader Vasconcelos and Deepak Ramachandran and Golnoosh Farnadi and Katherine Heller and Mohammad Havaei and Negar Rostamzadeh},
    booktitle={Proceedings of the IEEE/CVF International Conference on Computer Vision},
    pages={16420--16430},
    year={2025},
}

@INPROCEEDINGS{ren2025sixcd,
  author={Ren, Jie and Chen, Kangrui and Cui, Yingqian and Zeng, Shenglai and Liu, Hui and Xing, Yue and Tang, Jiliang and Lyu, Lingjuan},
  booktitle={2025 IEEE/CVF Conference on Computer Vision and Pattern Recognition (CVPR)}, 
  title={Six-CD: Benchmarking Concept Removals for Text-to-image Diffusion Models}, 
  year={2025},
  pages={28769-28778}
}

@article{bai2025qwen2,
  title={Qwen2.5-{VL} Technical Report},
  author={Bai, Shuai and Chen, Keqin and Liu, Xuejing and Wang, Jialin and Ge, Wenbin and Song, Sibo and Dang, Kai and Wang, Peng and Wang, Shijie and Tang, Jun and others},
  journal={arXiv preprint arXiv:2502.13923},
  year={2025}
}

@article{ho2020denoising,
  title={Denoising diffusion probabilistic models},
  author={Ho, Jonathan and Jain, Ajay and Abbeel, Pieter},
  journal={Advances in neural information processing systems},
  volume={33},
  pages={6840--6851},
  year={2020}
}

@inproceedings{patel2025learning,
  title={Learning to unlearn while retaining: Combating gradient conflicts in machine unlearning},
  author={Patel, Gaurav and Qiu, Qiang},
  booktitle={Proceedings of the IEEE/CVF International Conference on Computer Vision},
  pages={4211--4221},
  year={2025}
}

@article{zhang2024defensive,
  title={Defensive unlearning with adversarial training for robust concept erasure in diffusion models},
  author={Zhang, Yimeng and Chen, Xin and Jia, Jinghan and Zhang, Yihua and Fan, Chongyu and Liu, Jiancheng and Hong, Mingyi and Ding, Ke and Liu, Sijia},
  journal={Advances in neural information processing systems},
  volume={37},
  pages={36748--36776},
  year={2024}
}

@inproceedings{srivatsan2025stereo,
  title={Stereo: A two-stage framework for adversarially robust concept erasing from text-to-image diffusion models},
  author={Srivatsan, Koushik and Shamshad, Fahad and Naseer, Muzammal and Patel, Vishal M and Nandakumar, Karthik},
  booktitle={Proceedings of the Computer Vision and Pattern Recognition Conference},
  pages={23765--23774},
  year={2025}
}

@inproceedings{rombach2022high,
  title={High-resolution image synthesis with latent diffusion models},
  author={Rombach, Robin and Blattmann, Andreas and Lorenz, Dominik and Esser, Patrick and Ommer, Bj{\"o}rn},
  booktitle={Proceedings of the IEEE/CVF conference on computer vision and pattern recognition},
  pages={10684--10695},
  year={2022}
}

@inproceedings{
bui2025fantastic,
title={Fantastic Targets for Concept Erasure in Diffusion Models and Where To Find Them},
author={Anh Tuan Bui and Thuy-Trang Vu and Long Tung Vuong and Trung Le and Paul Montague and Tamas Abraham and Junae Kim and Dinh Phung},
booktitle={The Thirteenth International Conference on Learning Representations},
year={2025},
url={https://openreview.net/forum?id=tZdqL5FH7w}
}

@inproceedings{lu2025concepts,
  title={When Are Concepts Erased From Diffusion Models?},
  author={Kevin Lu and Nicky Kriplani and Rohit Gandikota and Minh Pham and David Bau and Chinmay Hegde and Niv Cohen},
  booktitle={39th Conference on Neural Information Processing Systems (NeurIPS)},
  year={2025}
}

@inproceedings{
mavrothalassitis2026ascent,
title={Ascent Fails to Forget},
author={Ioannis Mavrothalassitis and Pol Puigdemont and Noam Itzhak Levi and Volkan Cevher},
booktitle={The Thirty-ninth Annual Conference on Neural Information Processing Systems},
year={2025},
url={https://openreview.net/forum?id=KBJSV1XApq}
}

@inproceedings{li2025set,
  title={Set you straight: Auto-steering denoising trajectories to sidestep unwanted concepts},
  author={Li, Leyang and Lu, Shilin and Ren, Yan and Kong, Adams Wai-Kin},
  booktitle={Proceedings of the 33rd ACM International Conference on Multimedia},
  pages={9257--9266},
  year={2025}
}

@inproceedings{
    kusumba2025eraseflow,
    title={{EraseFlow}: Learning Concept Erasure Policies via {GFlowNet}-Driven Alignment},
    author={Naga Sai Abhiram Kusumba and Maitreya Patel and Kyle Min and Changhoon Kim and Chitta Baral and Yezhou Yang},
    booktitle={The Thirty-ninth Annual Conference on Neural Information Processing Systems},
    year={2025},
    url={https://openreview.net/forum?id=igB289kbej}
}

@inproceedings{gao2025meta,
  title={Meta-unlearning on diffusion models: Preventing relearning unlearned concepts},
  author={Gao, Hongcheng and Pang, Tianyu and Du, Chao and Hu, Taihang and Deng, Zhijie and Lin, Min},
  booktitle={Proceedings of the IEEE/CVF International Conference on Computer Vision},
  pages={2131--2141},
  year={2025}
}

@article{bengio2023gflownet,
  title={{GFlowNet} Foundations},
  author={Bengio, Yoshua and Lahlou, Salem and Deleu, Tristan and Hu, Edward J and Tiwari, Mo and Bengio, Emmanuel},
  journal={Journal of Machine Learning Research},
  volume={24},
  number={210},
  pages={1--55},
  year={2023}
}

@inproceedings{liu2025nablagfn,
  title={Efficient Diversity-Preserving Diffusion Alignment via Gradient-Informed {GFlowNets}},
  author={Liu, Zhen and Xiao, Tim Z. and Liu, Weiyang and Bengio, Yoshua and Zhang, Dinghuai},
  booktitle={International Conference on Learning Representations (ICLR)},
  year={2025}
}

@inproceedings{malkin2022trajectory,
  title={Trajectory Balance: Improved Credit Assignment in {GFlowNets}},
  author={Malkin, Nikolay and Jain, Moksh and Bengio, Emmanuel and Sun, Chen and Bengio, Yoshua},
  booktitle={Advances in Neural Information Processing Systems (NeurIPS)},
  year={2022}
}

@inproceedings{park2024duo,
  title={Direct Unlearning Optimization for Robust and Safe Text-to-Image Models},
  author={Park, Yong-Hyun and Yun, Sangdoo and Kim, Jin-Hwa and Kim, Junho and Jang, Geonhui and Jeong, Yonghyun and Jo, Junghyo and Lee, Gayoung},
  booktitle={Advances in Neural Information Processing Systems (NeurIPS)},
  year={2024}
}

@article{cho2025fade,
  title={Reference-Specific Unlearning Metrics Can Hide the Truth: A Reality Check},
  author={Cho, Sungjun and Hwang, Dasol and Sala, Frederic and Hwang, Sangheum and Cho, Kyunghyun and Cha, Sungmin},
  journal={arXiv preprint arXiv:2510.12981},
  year={2025}
}

@inproceedings{song2021scorebased,
  title={Score-Based Generative Modeling through Stochastic Differential Equations},
  author={Song, Yang and Sohl-Dickstein, Jascha and Kingma, Diederik P and Kumar, Abhishek and Ermon, Stefano and Poole, Ben},
  booktitle={International Conference on Learning Representations (ICLR)},
  year={2021}
}

@inproceedings{black2024ddpo,
  title={Training Diffusion Models with Reinforcement Learning},
  author={Black, Kevin and Janner, Michael and Du, Yilun and Kostrikov, Ilya and Levine, Sergey},
  booktitle={International Conference on Learning Representations (ICLR)},
  year={2024}
}

@inproceedings{clark2024draft,
  title={Directly Fine-Tuning Diffusion Models on Differentiable Rewards},
  author={Clark, Kevin and Vicol, Paul and Swersky, Kevin and Fleet, David J},
  booktitle={International Conference on Learning Representations (ICLR)},
  year={2024}
}

@article{uehara2024finetuning,
  title={Fine-Tuning of Continuous-Time Diffusion Models as Entropy-Regularized Control},
  author={Uehara, Masatoshi and Zhao, Yulai and Black, Kevin and Hajiramezanali, Ehsan and Scalia, Gabriele and Diamant, Nathaniel Lee and Tseng, Alex M. and Biancalani, Tommaso and Levine, Sergey},
  journal={arXiv preprint arXiv:2402.15194},
  year={2024}
}

@inproceedings{hu2022lora,
  title={{LoRA}: Low-Rank Adaptation of Large Language Models},
  author={Hu, Edward J and Shen, Yelong and Wallis, Phillip and Allen-Zhu, Zeyuan and Li, Yuanzhi and Wang, Shean and Wang, Lu and Chen, Weizhu},
  booktitle={International Conference on Learning Representations (ICLR)},
  year={2022}
}

@inproceedings{radford2021clip,
  title={Learning Transferable Visual Models from Natural Language Supervision},
  author={Radford, Alec and Kim, Jong Wook and Hallacy, Chris and Ramesh, Aditya and Goh, Gabriel and Agarwal, Sandhini and Sastry, Girish and Askell, Amanda and Mishkin, Pamela and Clark, Jack and others},
  booktitle={International Conference on Machine Learning (ICML)},
  year={2021}
}

@inproceedings{heusel2017fid,
  title={{GANs} Trained by a Two Time-Scale Update Rule Converge to a Local Nash Equilibrium},
  author={Heusel, Martin and Ramsauer, Hubert and Unterthiner, Thomas and Nessler, Bernhard and Hochreiter, Sepp},
  booktitle={Advances in Neural Information Processing Systems (NeurIPS)},
  year={2017}
}

@misc{gdpr2016,
  title={Regulation ({EU}) 2016/679 of the European Parliament and of the Council of 27 April 2016},
  author={{European Parliament and Council of the European Union}},
  year={2016},
  howpublished={Official Journal of the European Union},
  note={General Data Protection Regulation, Article 17: Right to erasure},
  url={https://eur-lex.europa.eu/legal-content/EN/TXT/?uri=CELEX:32016R0679}
}

@inproceedings{
zhao2024what,
title={What makes unlearning hard and what to do about it},
author={Kairan Zhao and Meghdad Kurmanji and George-Octavian B{\u{a}}rbulescu and Eleni Triantafillou and Peter Triantafillou},
booktitle={The Thirty-eighth Annual Conference on Neural Information Processing Systems},
year={2024},
url={https://openreview.net/forum?id=QAbhLBF72K}
}

@inproceedings{wu2025erasing,
  title={Erasing undesirable influence in diffusion models},
  author={Wu, Jing and Le, Trung and Hayat, Munawar and Harandi, Mehrtash},
  booktitle={Proceedings of the Computer Vision and Pattern Recognition Conference},
  pages={28263--28273},
  year={2025}
}

\clearpage

\appendix

\raggedbottom
\renewcommand{\baselinestretch}{1.07}\selectfont
\setlength{\parindent}{0pt}
\setlength{\parskip}{0.36em plus 0.1em minus 0.05em}
\setlength{\abovedisplayskip}{0.95em plus 0.28em minus 0.18em}
\setlength{\belowdisplayskip}{0.95em plus 0.28em minus 0.18em}
\setlength{\abovedisplayshortskip}{0.7em plus 0.22em minus 0.12em}
\setlength{\belowdisplayshortskip}{0.78em plus 0.22em minus 0.12em}
\setlength{\jot}{0.35em}
\setlist[itemize]{leftmargin=1.65em,itemsep=0.35em,topsep=0.42em}
\setlist[enumerate]{leftmargin=1.75em,itemsep=0.48em,topsep=0.55em}

\phantomsection
\begin{center}
{\Large\bfseries Appendix}\\[0.45em]
\end{center}
{\normalsize\bfseries Table of Contents}
\label{app:toc}

% \vspace{0.25em}
\noindent\textcolor{appendixrule}{\rule{\linewidth}{0.4pt}}
% \vspace{0.35em}

\AppendixTocSection{app:algorithm}{Algorithm Pseudocode}

\AppendixTocSection{app:background}{Diffusion and $\nabla$-GFlowNet Background}
\AppendixTocSubsection{sec:prelim_diffusion}{Diffusion Models}
\AppendixTocSubsection{sec:prelim_gfn}{GFlowNets and Detailed Balance}
\AppendixTocSubsection{sec:prelim_nabla}{Score-Based Detailed Balance}

\AppendixTocSection{app:proofs}{Theoretical Analysis and Proofs}
\AppendixTocSubsection{app:residual_db_losses}{Residual $\nabla$-DB Losses}
\AppendixTocSubsection{app:distributional_correctness}{Distributional Correctness}
\AppendixTocSubsection{app:db_score_correspondence}{DB Score Matching Correspondence}
\AppendixTocSubsection{app:reverse_gaussian_score}{Reverse Gaussian Score}
\AppendixTocSubsection{app:forget_energy_adaptation}{Forget-Energy Adaptation}
\AppendixTocSubsection{app:optimization_caveat}{Optimization Caveat}

\AppendixTocSection{app:derivations}{Extended Derivations}
\AppendixTocSubsection{app:gibbs}{Closed-Form Conditional Gibbs Projection}
\AppendixTocSubsection{app:gaussian}{Gaussian Instantiation of $\nabla$-DB}
\AppendixTocSubsection{app:correction}{Non-Ideal Model Correction}

\AppendixTocSection{app:exp_details}{Additional Experimental Details}
\AppendixTocSubsection{app:clip_model}{CLIP Model}
\AppendixTocSubsection{app:base_model}{Base Diffusion Model and Adapters}
\AppendixTocSubsection{app:prompt_sets}{Prompt Sets and Generation Strategy}
\AppendixTocSubsection{app:optimization}{Optimization and Schedule}
\AppendixTocSubsection{app:compute}{Compute}
\AppendixTocSubsection{app:hyperparameter_sensitivity}{Hyperparameter Sensitivity}
\AppendixTocSubsection{app:eval_protocol}{Evaluation Protocol}

\AppendixTocSection{app:architecture_results}{Performance Across Multiple Diffusion Architectures}
\AppendixTocSection{app:per_concept_results}{Per-Concept Quantitative Results}

\vspace{0.35em}
\noindent\textcolor{appendixrule}{\rule{\linewidth}{0.4pt}}

\clearpage

\section{Algorithm Pseudocode}
\label{app:algorithm}

\begin{algorithm}[ht]
\caption{Forget-Energy-Tilted $\nabla$-GFlowNet Training}
\label{alg:training}
\begingroup
\renewcommand{\baselinestretch}{1.08}\selectfont
\begin{algorithmic}[1]
\setlength{\itemsep}{0.22em}
\REQUIRE Pretrained denoiser $\beps^{\rm pre}$; CLIP model; concept $\mathcal{C}$; descriptors $\mathcal{Q}_{\mathcal{C}}$; forget prompts $\mathcal{Y}_{\mathcal{C}}$; hyperparameters $\tau,\alpha,\lambda_{\text{scale}}$
\STATE \textbf{Setup.} Initialize LoRA parameters $\Delta\theta$ on U-Net attention layers $(W_Q,W_K,W_V,W_O)$ with rank $r=8$.
\STATE Initialize the residual parameterization network $g_\phi$ with channels $\{64,128,256,256\}$.
\STATE Initialize the adapted denoiser $\beps_\theta \leftarrow \beps^{\rm pre}+\Delta\theta$.
\STATE \vspace{0.2em}
\FOR{epoch $= 1, \ldots, N_{\text{epochs}}$}
    \FOR{batch $= 1, \ldots, N_{\text{batches}}$}
        \STATE Sample forget prompts $\{y_i\}$ from $\mathcal{Y}_{\mathcal{C}}$
        \STATE \vspace{0.15em}
        \STATE \textbf{Phase 1: sample prompt-conditioned trajectories.}
        \STATE Sample $\bz_T \sim \mathcal{N}(\mathbf{0}, \mathbf{I})$
        \FOR{$t = T, T-1, \ldots, 1$}
            \STATE $\beps_\theta(\bz_t, t, y_i) \gets$ LoRA-adapted U-Net prediction
            \STATE $\bz_{t-1} \sim P_F^\theta(\cdot \mid \bz_t,y_i)$ via one DDPM denoising step
        \ENDFOR
        \STATE Decode terminal latent: $\hat{\bx} \gets \mathcal{D}(\bz_0)$
        \STATE \vspace{0.15em}
        \STATE \textbf{Phase 2: evaluate the terminal forget energy.}
        \STATE $s_{\mathcal{C}}(\hat{\bx}) \gets |\mathcal{Q}_{\mathcal{C}}|^{-1}\sum_{q\in\mathcal{Q}_{\mathcal{C}}}\operatorname{sim}_{\mathrm{CLIP}}(\hat{\bx},q)$
        \STATE $m_{\mathcal{C}}(\hat{\bx}) \gets \max(s_{\mathcal{C}}(\hat{\bx})-\tau,0)$
        \STATE $\widetilde{E}_{\mathcal{C}}(\hat{\bx}) \gets \lambda_{\text{scale}}\big(\exp(\alpha\,m_{\mathcal{C}}(\hat{\bx}))-1\big)$
        \STATE \vspace{0.15em}
        \STATE \textbf{Phase 3: match residual $\nabla$-DB scores.}
        \STATE Sample a timestep subset $\mathcal{T}\subset\{1,\ldots,T\}$ with $|\mathcal{T}|=0.1T$
        \FOR{$t \in \mathcal{T}$}
            \STATE $\mathbf{s}_{\mathrm{fwd}} \gets \big(\mu_\theta(\bz_t,t,y_i)-\mu^{\rm pre}(\bz_t,t,y_i)\big)/\sigma_t^2$
            \STATE $\mathbf{s}_{\mathrm{rev}} \gets \nabla_{\bz_t}\log P_F^{\rm pre}(\bz_{t-1}\mid\bz_t,y_i) - \nabla_{\bz_t}\log P_F^\theta(\bz_{t-1}\mid\bz_t,y_i)$
            \STATE $\mathbf{h}_{t-1} \gets -\bar\alpha_{t-1}\nabla_{\bz_{t-1}}\widetilde{E}_{\mathcal{C}}(\hat{\bx}_\theta(\bz_{t-1},t-1,y_i)) + g_\phi(\bz_{t-1},t-1)$
            \STATE $\mathbf{h}_{t} \gets -\bar\alpha_t\nabla_{\bz_t}\widetilde{E}_{\mathcal{C}}(\hat{\bx}_\theta(\bz_t,t,y_i)) + g_\phi(\bz_t,t)$
            \STATE Accumulate $\mathcal{L}_{\text{fwd}}$ from $\|\mathbf{s}_{\mathrm{fwd}}-\mathbf{h}_{t-1}\|^2$ and $\mathcal{L}_{\text{rev}}$ from $\|\mathbf{s}_{\mathrm{rev}}-\mathbf{h}_{t}\|^2$
        \ENDFOR
        \STATE \vspace{0.15em}
        \STATE \textbf{Phase 4: update parameters.}
        \STATE $\mathcal{L}_{\text{term}} \gets \|g_\phi(\bz_0,0)\|^2$ \COMMENT{data-end boundary}
        \STATE $\mathcal{L} \gets \mathcal{L}_{\text{fwd}} + \lambda_{\text{rev}} \mathcal{L}_{\text{rev}} + \lambda_{\text{term}} \mathcal{L}_{\text{term}}$
        \STATE Update $\Delta\theta$ and $\phi$ with AdamW
        \STATE \vspace{0.2em}
    \ENDFOR
\ENDFOR
\RETURN LoRA-adapted model $\beps^{\rm pre} + \Delta\theta$
\end{algorithmic}
\endgroup
\end{algorithm}

\clearpage

\section{Background on Diffusion Models and \texorpdfstring{$\nabla$}{nabla}-GFlowNets}
\label{app:background}

This section reviews the two technical foundations underlying our method: (i) diffusion models, which provide the generative model and the score-based view of denoising transitions; and (ii) GFlowNets and their continuous score-matching variant $\nabla$-GFlowNets, which provide the optimization framework for tilting the diffusion sampler toward an unnormalized target. The presentation is intended to be self-contained for a reader familiar with general deep generative modeling but not necessarily fluent in both subfields, and to fix the notation used by the proofs in Appendix~\ref{app:proofs} and Appendix~\ref{app:derivations}.

\subsection{Diffusion Models}
\label{sec:prelim_diffusion}

Denoising diffusion probabilistic models (DDPMs)~\citep{ho2020denoising,song2021scorebased} learn a generative model by reversing a fixed corruption process. The corruption process gradually adds Gaussian noise to data until samples are indistinguishable from $\mathcal{N}(\mathbf{0},\mathbf{I})$; the generative model is then a neural network trained to invert this corruption one step at a time. Sampling proceeds by drawing pure noise and iteratively denoising back to a clean datum.

\paragraph{Forward (noising) process.}
Let $\bz_0$ be a clean latent. The forward process is a fixed Markov chain governed by a variance schedule $\{\beta_t\}_{t=1}^T$, with $\alpha_t = 1-\beta_t$ and cumulative product $\bar\alpha_t=\prod_{s=1}^t\alpha_s$. Each per-step kernel is Gaussian:
\begin{align}
q(\bz_t \mid \bz_{t-1}) &= \mathcal{N}\!\left(\sqrt{\alpha_t}\,\bz_{t-1},\; (1-\alpha_t)\mathbf{I}\right).
\end{align}
A standard property of this chain is that $\bz_t$ given $\bz_0$ is itself Gaussian:
\begin{align}
\label{eq:forward_marginal}
q(\bz_t \mid \bz_0) &= \mathcal{N}\!\left(\sqrt{\bar\alpha_t}\,\bz_0,\;(1-\bar\alpha_t)\mathbf{I}\right),
\end{align}
which lets us directly sample any noisy intermediate without rolling out the chain. The schedule is chosen so that $\bar\alpha_T \approx 0$, making $q(\bz_T)$ essentially indistinguishable from a standard Gaussian.

\paragraph{Reverse (denoising) process.}
Generation requires sampling from $q(\bz_{t-1} \mid \bz_t)$, which is intractable in closed form. The model parameterizes a Gaussian approximation
\begin{align}
p_\theta(\bz_{t-1} \mid \bz_t,y) &= \mathcal{N}\!\left(\mu_\theta(\bz_t,t,y),\;\sigma_t^2\mathbf{I}\right),
\end{align}
where $y$ is an optional conditioning input (e.g., a text prompt) and $\sigma_t^2$ is fixed by the schedule. Using the standard $\bz_0$-prediction reparameterization, the predicted mean is rewritten in terms of a learned noise predictor $\beps_\theta$:
\begin{align}
\label{eq:mean_reparam}
\mu_\theta(\bz_t,t,y) &= \frac{1}{\sqrt{\alpha_t}}\!\left(\bz_t - \frac{1-\alpha_t}{\sqrt{1-\bar\alpha_t}}\beps_\theta(\bz_t,t,y)\right).
\end{align}
Throughout the main text, $P_F^\theta(\bz_{t-1}\mid\bz_t,y)$ denotes $p_\theta(\bz_{t-1}\mid\bz_t,y)$ in GFlowNet notation; the subscript $F$ refers to the (forward) sampling direction of the GFlowNet, not to the noising direction.

\paragraph{Training objective.}
The denoiser is trained by minimizing the simplified denoising score-matching loss
\begin{align}
\mathcal{L}_{\text{DDPM}}
&=
\E_{\bz_0,\beps,t}\!\left[
\left\|\beps - \beps_\theta\!\left(\sqrt{\bar\alpha_t}\bz_0+\sqrt{1-\bar\alpha_t}\beps,\, t,\, y\right)\right\|^2
\right],
\end{align}
which is, up to per-timestep weighting, equivalent to maximizing a variational lower bound on $\log p_\theta(\bz_0)$ and to a weighted denoising score-matching objective. Our method does \emph{not} re-train this objective; we keep the pretrained $\beps^{\rm pre}$ frozen and learn only a low-rank correction $\Delta\theta$ on top.

\paragraph{Score-based view.}
The crucial connection used by our method is between the noise predictor $\beps_\theta$ and the score function of the marginal $p_t(\bz_t\mid y)$:
\begin{align}
\label{eq:score_noise}
\nabla_{\bz_t} \log p_t(\bz_t \mid y) &= -\frac{\beps_\theta(\bz_t, t, y)}{\sqrt{1 - \bar\alpha_t}}.
\end{align}
That is, the noise predictor and the marginal score differ only by a deterministic, time-dependent rescaling. As a consequence, a diffusion U-Net is, ``for free,'' a score model. Any modification we wish to apply to the sampler can be expressed as a modification to the per-step score, which is exactly the form on which $\nabla$-GFlowNet operates.

\paragraph{One-step clean-image estimate.}
At any timestep $t$, the noise predictor implies a clean-image estimate via Eq.~\eqref{eq:forward_marginal}:
\begin{align}
\hat{\bx}_\theta(\bz_t,t,y)
&=
\mathcal{D}\!\left(\frac{\bz_t-\sqrt{1-\bar\alpha_t}\,\beps_\theta(\bz_t,t,y)}{\sqrt{\bar\alpha_t}}\right).
\end{align}
Here $\mathcal{D}$ is the latent decoder; in pixel diffusion, $\mathcal{D}$ is the identity. This $\hat{\bx}_\theta$ is what allows us to evaluate the terminal forget energy $\widetilde{E}_{\mathcal{C}}$ from a noisy intermediate state, and it is the object differentiated in the forward-looking flow score of Eq.~\eqref{eq:forward_looking}.

\paragraph{Latent diffusion models.}
Stable Diffusion-style latent diffusion models (LDMs)~\citep{rombach2022high} run the diffusion process in the latent space of a pretrained encoder--decoder pair $(\mathcal{E},\mathcal{D})$ rather than at pixel resolution. Latents $\bz_t$ are spatially much smaller than the corresponding image, and decoded images are obtained as $\bx = \mathcal{D}(\bz_0)$. Conditioning on text prompts $y$ enters through cross-attention layers in the U-Net. Our experiments use Stable Diffusion v1.5, and all derivations refer to latents $\bz_t$ unless otherwise stated.

\paragraph{Why this matters for unlearning.}
Two properties of diffusion models make them a particularly clean substrate for our method:
\begin{itemize}[leftmargin=1.6em,itemsep=0.35em,topsep=0.45em]
\item \textbf{Tractable per-step Gaussians.} Both the pretrained transition $P_F^{\rm pre}$ and the LoRA-adapted transition $P_F^\theta$ are Gaussians with the same variance, so all log-density gradients used in the residual $\nabla$-DB losses are available in closed form (Appendix~\ref{app:gaussian}).
\item \textbf{Score representation by construction.} Any per-step modification to the sampler can be expressed entirely through modifications to the noise predictor, which is also a score model via Eq.~\eqref{eq:score_noise}.
\end{itemize}

\subsection{GFlowNets and the Detailed Balance Condition}
\label{sec:prelim_gfn}

\paragraph{What problem do GFlowNets solve?}
Generative Flow Networks (GFlowNets)~\citep{bengio2023gflownet} are amortized samplers for unnormalized reward distributions. Given a non-negative reward $R(\bx)\ge 0$, the goal is to learn a policy whose terminal-state marginal satisfies
\begin{align}
P_F^\top(\bx) &\propto R(\bx),
\end{align}
where $P_F^\top$ denotes the marginal at the end of the sampling chain. The defining feature is that this proportionality is achieved without ever computing the partition function $Z=\int R(\bx)\,d\bx$. Once trained, sampling requires a single forward pass through the policy, in contrast to MCMC methods that require per-sample iteration.

\paragraph{Trajectory and flow view.}
A GFlowNet builds a terminal sample through a sequence of intermediate states $s_0\to s_1\to\cdots\to s_T = \bx$, governed by a forward (sampling) policy $P_F(s_{t+1}\mid s_t)$ and a backward policy $P_B(s_t\mid s_{t+1})$. The flow function $F(s_t)\ge 0$ assigns a non-negative weight to each state and plays the role of an unnormalized marginal: intuitively, $F(s_t)$ measures the total reward reachable from $s_t$. In the diffusion setting, the trajectory is the denoising chain $\bz_T\to\cdots\to\bz_0$, $P_F$ is the denoising kernel, and $P_B$ is the (fixed) DDPM noising kernel.

\paragraph{Detailed Balance.}
The \emph{Detailed Balance} (DB) objective is the most local of the standard GFlowNet objectives, in that it constrains a single transition rather than a full trajectory:
\begin{align}
\label{eq:db_appendix}
P_F(s_{t+1}\mid s_t)\,F(s_t) &= P_B(s_t\mid s_{t+1})\,F(s_{t+1}),
\qquad
F(s_T) = R(s_T).
\end{align}
The terminal boundary condition $F(s_T)=R(s_T)$ ties the flow to the reward at the data end of the chain. Whenever DB holds at every transition together with this boundary condition, a telescoping argument cancels the intermediate $F$-terms along any trajectory, leaving only the boundary contributions, and yields the desired proportional-sampling property $P_F^\top(\bx)\propto R(\bx)$.

\paragraph{Why does the partition function cancel?}
DB only requires \emph{ratios} of policy densities and flows. The unknown partition function $Z$ enters $F$ only as an overall multiplicative constant, which appears in both $F(s_t)$ and $F(s_{t+1})$ and therefore cancels in any pairwise ratio. This is why GFlowNets can sample from an unnormalized target without ever evaluating $Z$ -- exactly the property our unlearning target $p_{\rm pre}(\bx\mid y)\exp(-\widetilde{E}_{\mathcal{C}}(\bx))$ requires, since the prompt-dependent partition function is intractable.

\paragraph{How GFlowNets relate to MCMC and standard generative modeling.}
GFlowNets occupy a middle ground: like MCMC, they target an unnormalized density; like standard generative models, they amortize sampling into a single forward pass at inference time. This combination is precisely what concept unlearning at inference-time requires: a tilted diffusion sampler that runs at the same per-sample cost as the original.

\subsection{\texorpdfstring{$\nabla$}{nabla}-GFlowNet: Score-Based Detailed Balance}
\label{sec:prelim_nabla}

\paragraph{Why standard DB is hard for diffusion.}
Direct enforcement of Eq.~\eqref{eq:db_appendix} requires evaluating $\log P_F$ and $\log P_B$ pointwise. For continuous, high-dimensional state spaces such as diffusion latents, log-densities of learned transitions can be expensive or unstable to evaluate, especially when $P_F$ is parameterized by a large neural network. This makes pure DB a poor match for diffusion-based samplers.

\paragraph{The $\nabla$ trick.}
\citet{liu2025nablagfn} observe that one can sidestep log-densities by differentiating the DB condition with respect to the state, converting a log-probability balance into a \emph{score} balance. Taking $\log$ of Eq.~\eqref{eq:db_appendix} and then $\nabla_{s_{t+1}}$, and using that $F(s_t)$ does not depend on $s_{t+1}$, gives the $\nabla$-DB equation:
\begin{align}
\label{eq:nabla_db}
\nabla_{s_{t+1}}\log P_F(s_{t+1}\mid s_t)
&=
\nabla_{s_{t+1}}\log P_B(s_t\mid s_{t+1})
+
\nabla_{s_{t+1}}\log F(s_{t+1}).
\end{align}
The same condition can be obtained by differentiating with respect to $s_t$, yielding a complementary reverse-score equation; see Appendix~\ref{app:residual_db_losses}. Crucially, all three terms are scores rather than log-densities, so the equation can be enforced as an $\ell_2$-style score-matching loss using exactly the gradient information that diffusion models already expose through Eq.~\eqref{eq:score_noise}.

\paragraph{Residual variant.}
The full DB condition asks for a flow $F$ that is consistent with the absolute target reward. For unlearning, however, we are not changing the target distribution from scratch; we are tilting an already-trained diffusion model. It is therefore natural to factor out the pretrained sampler and learn only the residual ratio
\begin{align}
\tilde{F}(\bz_t,y) &= \frac{F(\bz_t,y)}{F^{\rm pre}(\bz_t,y)},
\qquad
\tilde{P}_F^\theta(\bz_{t-1}\mid\bz_t,y) = \frac{P_F^\theta(\bz_{t-1}\mid\bz_t,y)}{P_F^{\rm pre}(\bz_{t-1}\mid\bz_t,y)},
\end{align}
which is exactly the residual $\nabla$-DB condition stated as Eq.~\eqref{eq:residual_ratio_balance} in the main text. Subtracting the pretrained instance of Eq.~\eqref{eq:nabla_db} from the finetuned instance cancels the (fixed) noising kernel $P_B$ and yields a clean residual score equation; the formal derivation is in Appendix~\ref{app:residual_db_losses}.

\paragraph{Why this is the right tool for our problem.}
Three properties line up with concept unlearning:
\begin{itemize}[leftmargin=1.6em,itemsep=0.35em,topsep=0.45em]
\item \textbf{Native to diffusion.} The optimization works in score space, which is the natural representation of a diffusion U-Net.
\item \textbf{Unnormalized targets.} The flow function absorbs the unknown partition function, so the unnormalized energy $\widetilde{E}_{\mathcal{C}}$ is sufficient.
\item \textbf{Per-transition local objective.} Each timestep contributes its own loss, so optimization decomposes over the chain and is compatible with mini-batch training.
\end{itemize}
Together, these motivate our use of residual $\nabla$-GFlowNet as the optimization backbone (Section~\ref{sec:optimization}); the unlearning-specific choices -- the prompt-conditioned tilted target, the thresholded forget energy, and the forward-looking flow parameterization -- are layered on top of this backbone.

\clearpage

\section{Theoretical Analysis and Proofs}
\label{app:proofs}

\subsection{From Detailed Balance to the Trainable Residual \texorpdfstring{$\nabla$}{nabla}-DB Losses}
\label{app:residual_db_losses}

For a fixed prompt $y$, let $P_B(\bz_t | \bz_{t-1})$ denote the fixed DDPM noising transition, while $P_F$ denotes the denoising transition in the GFlowNet sampling direction.
The DDPM denoising transition $\bz_t \to \bz_{t-1}$ satisfies the detailed-balance relation
\begin{align}
\label{eq:db_ddpm}
P_F(\bz_{t-1} | \bz_t,y)\,F(\bz_t,y) &= P_B(\bz_t | \bz_{t-1})\,F(\bz_{t-1},y).
\end{align}
Taking logarithms gives
\begin{align}
\log P_F(\bz_{t-1} | \bz_t,y) + \log F(\bz_t,y)
&=
\log P_B(\bz_t | \bz_{t-1}) + \log F(\bz_{t-1},y).
\end{align}
Differentiating with respect to the next denoised state $\bz_{t-1}$ yields the forward $\nabla$-DB equation
\begin{align}
\nabla_{\bz_{t-1}} \log P_F(\bz_{t-1} | \bz_t,y)
&=
\nabla_{\bz_{t-1}} \log P_B(\bz_t | \bz_{t-1})
+
\nabla_{\bz_{t-1}} \log F(\bz_{t-1},y),
\end{align}
because $F(\bz_t,y)$ does not depend on $\bz_{t-1}$.
Differentiating instead with respect to the conditioning latent $\bz_t$ yields the reverse equation
\begin{align}
\nabla_{\bz_t} \log P_B(\bz_t | \bz_{t-1})
-
\nabla_{\bz_t} \log P_F(\bz_{t-1} | \bz_t,y)
&=
\nabla_{\bz_t} \log F(\bz_t,y),
\end{align}
because $F(\bz_{t-1},y)$ does not depend on $\bz_t$.

Now let $P_F^{\rm pre}$ be the pretrained denoising policy and $P_F^\theta$ the finetuned one, both conditioned on the same prompt $y$.
Subtracting the pretrained and finetuned forward equations cancels the fixed noising policy $P_B$ and gives
\begin{align}
\nabla_{\bz_{t-1}} \log P_F^\theta(\bz_{t-1} | \bz_t,y)
-
\nabla_{\bz_{t-1}} \log P_F^{\rm pre}(\bz_{t-1} | \bz_t,y)
&=
\nabla_{\bz_{t-1}} \log \tilde{F}(\bz_{t-1},y),
\end{align}
with $\tilde{F}=F/F^{\rm pre}$.
Subtracting the reverse equations gives
\begin{align}
\nabla_{\bz_t} \log P_F^{\rm pre}(\bz_{t-1} | \bz_t,y)
-
\nabla_{\bz_t} \log P_F^\theta(\bz_{t-1} | \bz_t,y)
&=
\nabla_{\bz_t} \log \tilde{F}(\bz_t,y).
\end{align}
These two equations are exactly the targets enforced by $\mathcal{L}_{\text{fwd}}$ and $\mathcal{L}_{\text{rev}}$.

With the forward-looking parameterization from Eq.~\eqref{eq:forward_looking}, the same residual flow network is evaluated at the state appearing on the right-hand side:
\begin{itemize}
    \item forward direction: evaluate $\nabla \log \tilde{F}$ at $(\bz_{t-1},y)$;
    \item reverse direction: evaluate $\nabla \log \tilde{F}$ at $(\bz_t,y)$.
\end{itemize}
Therefore the practical losses are
\begin{align}
\label{eq:loss_forward}
\mathcal{L}_{\text{fwd}}
&=
\E\!\left[
\left\|
\nabla_{\bz_{t-1}} \log P_F^\theta(\bz_{t-1} | \bz_t,y)
-
\nabla_{\bz_{t-1}} \log P_F^{\rm pre}(\bz_{t-1} | \bz_t,y)
-
\nabla_{\bz_{t-1}} \log \tilde{F}(\bz_{t-1},y)
\right\|^2
\right], \\
\label{eq:loss_reverse}
\mathcal{L}_{\text{rev}}
&=
\E\!\left[
\left\|
\nabla_{\bz_t} \log P_F^{\rm pre}(\bz_{t-1} | \bz_t,y)
-
\nabla_{\bz_t} \log P_F^\theta(\bz_{t-1} | \bz_t,y)
-
\nabla_{\bz_t} \log \tilde{F}(\bz_t,y)
\right\|^2
\right].
\end{align}
The terminal penalty $\mathcal{L}_{\text{term}}=\E_{\bz_0,y}[\|g_\phi(\bz_0,0)\|^2]$ enforces the data-end boundary condition.

\paragraph{How the gradients are computed in practice.}
The forward residual score has a closed form for Gaussian transitions, which yields Eq.~\eqref{eq:loss_forward}.
The reverse score differentiates the transition log-density with respect to the conditioning variable, so it depends on the Jacobian of the denoiser mean.
For
\[
P_F^\theta(\bz_{t-1} | \bz_t,y) = \mathcal{N}(\mu_\theta(\bz_t,t,y), \sigma_t^2 \mathbf I),
\]
we have
\begin{align}
\nabla_{\bz_t}\log P_F^\theta(\bz_{t-1} | \bz_t,y)
&=
\frac{1}{\sigma_t^2}J_{\mu_\theta}(\bz_t,t,y)^\top(\bz_{t-1} - \mu_\theta(\bz_t,t,y)).
\end{align}
The pretrained term is obtained by replacing $\mu_\theta$ with $\mu^{\rm pre}$.
These Jacobian-vector products are computed with standard autodiff and do not require forming a dense Jacobian matrix.

\subsection{Distributional Correctness}
\label{app:distributional_correctness}

The telescoping argument must be stated in integrated ratio form.
A score identity of the form $\nabla \log \tilde{P}_F^\theta = \nabla \log \tilde{F}$ only determines $\log \tilde{P}_F^\theta$ up to an additive function of the conditioning state, so it is not by itself sufficient to justify telescoping across time.
We therefore state the distributional theorem directly for the residual transition ratio.
Eq.~\eqref{eq:residual_ratio_db} is the multiplicative residual analogue of the detailed-balance factorization in \citet{bengio2023gflownet} and the integrated version of the residual result stated by \citet[Proposition~4]{liu2025nablagfn}.

\begin{theorem}[Distributional Correctness]
\label{thm:distributional}
For every fixed prompt $y$, let $P_F^{\rm pre}$ denote the pretrained denoising policy, let $P_F^\theta$ denote the finetuned denoising policy, and let $E(\bx)\ge 0$ be the terminal forget energy used for unlearning.
Define the residual transition ratio and residual flow by
\begin{align}
\tilde{P}_F^\theta(\bz_{t-1} \mid \bz_t,y) &\coloneqq \frac{P_F^\theta(\bz_{t-1} \mid \bz_t,y)}{P_F^{\rm pre}(\bz_{t-1} \mid \bz_t,y)},
\qquad
\tilde{F}(\bz_t,y) \coloneqq \frac{F(\bz_t,y)}{F^{\rm pre}(\bz_t,y)}.
\end{align}
Assume that for every denoising step $t=1,\dots,T$,
\begin{align}
\label{eq:residual_ratio_db}
\tilde{P}_F^\theta(\bz_{t-1} \mid \bz_t,y) &= \frac{\tilde{F}(\bz_{t-1},y)}{\tilde{F}(\bz_t,y)},
\end{align}
that the pretrained and finetuned models share the same initial noise prior $\bz_T\sim\mathcal{N}(\mathbf{0},\mathbf{I})$, and that at the data end $\bz_0$ the residual terminal flow satisfies
\begin{align}
\tilde{F}(\bz_0,y) &\propto \exp\!\big(-E(\bx)\big), \qquad \bx = \mathcal{D}(\bz_0),
\end{align}
while the initial residual flow is constant in $\bz_T$ on the support of the shared noise prior, i.e., there exists $c_{\mathrm{init}}(y) > 0$ such that
\begin{align}
\tilde{F}(\bz_T,y) &= c_{\mathrm{init}}(y).
\end{align}
Then:
\begin{align}
p_\theta(\bx \mid y) &\propto p_{\mathrm{pre}}(\bx \mid y) \cdot \exp\!\big(-E(\bx)\big).
\end{align}
\end{theorem}

\begin{proof}
Because the pretrained and finetuned models share the same initial noise prior, the ratio of path probabilities along a denoising trajectory $\bz_T \to \bz_{T-1} \to \cdots \to \bz_0$ is
\begin{align}
\frac{p_\theta(\bz_{0:T}\mid y)}{p_{\mathrm{pre}}(\bz_{0:T}\mid y)}
&=
\prod_{t=1}^{T}
\frac{P_F^\theta(\bz_{t-1} \mid \bz_t,y)}{P_F^{\rm pre}(\bz_{t-1} \mid \bz_t,y)}
=
\prod_{t=1}^{T}
\tilde{P}_F^\theta(\bz_{t-1} \mid \bz_t,y).
\end{align}
Applying Eq.~\eqref{eq:residual_ratio_db} gives
\begin{align}
\frac{p_\theta(\bz_{0:T}\mid y)}{p_{\mathrm{pre}}(\bz_{0:T}\mid y)}
&=
\prod_{t=1}^{T}
\frac{\tilde{F}(\bz_{t-1},y)}{\tilde{F}(\bz_t,y)}
=
\frac{\tilde{F}(\bz_0,y)}{\tilde{F}(\bz_T,y)}
\propto
\exp\!\big(-E(\bx)\big),
\end{align}
where the last proportionality uses the assumptions that $\tilde{F}(\bz_0,y) \propto \exp(-E(\bx))$ and $\tilde{F}(\bz_T,y)=c_{\mathrm{init}}(y)$, with $c_{\mathrm{init}}(y)$ independent of $\bz_T$.
The right-hand side depends on the trajectory only through the terminal clean sample $\bx$, so marginalizing over intermediate latents preserves the same proportionality for the image distribution:
\begin{align}
\frac{p_\theta(\bx\mid y)}{p_{\mathrm{pre}}(\bx\mid y)}
&\propto
\exp\!\big(-E(\bx)\big).
\end{align}
Rearranging gives the claim.
\end{proof}

\begin{remark}[What score matching alone guarantees]
Zero score loss enforces equality of gradients of log transition ratios.
This only determines $\log \tilde{P}_F^\theta(\bz_{t-1} \mid \bz_t,y)$ up to an additive function of $(\bz_t,y)$, so the theorem above is intentionally stated in the stronger integrated ratio form rather than inferred directly from gradient equality.
\end{remark}

\subsection{DB Score Matching Correspondence}
\label{app:db_score_correspondence}

\begin{proposition}
\label{prop:db_score}
For Gaussian denoising transitions
\[
P_F^\theta(\bz_{t-1} \mid \bz_t,y) = \mathcal{N}(\mu_\theta(\bz_t, t,y), \sigma_t^2 \mathbf{I}),
\qquad
P_F^{\rm pre}(\bz_{t-1} \mid \bz_t,y) = \mathcal{N}(\mu^{\rm pre}(\bz_t, t,y), \sigma_t^2 \mathbf{I}),
\]
the residual forward score is
\begin{align}
\nabla_{\bz_{t-1}} \log \tilde{P}_F^\theta(\bz_{t-1} \mid \bz_t,y)
&=
\nabla_{\bz_{t-1}} \log P_F^\theta(\bz_{t-1} \mid \bz_t,y)
-
\nabla_{\bz_{t-1}} \log P_F^{\rm pre}(\bz_{t-1} \mid \bz_t,y) \\
&=
\frac{\mu_\theta(\bz_t, t,y) - \mu^{\rm pre}(\bz_t, t,y)}{\sigma_t^2}.
\end{align}
\end{proposition}

\begin{proof}
The Gaussian scores with respect to the next denoised state $\bz_{t-1}$ are
\begin{align}
\nabla_{\bz_{t-1}} \log P_F^\theta &= -\frac{\bz_{t-1} - \mu_\theta(\bz_t, t,y)}{\sigma_t^2}, \qquad \nabla_{\bz_{t-1}} \log P_F^{\rm pre} = -\frac{\bz_{t-1} - \mu^{\rm pre}(\bz_t, t,y)}{\sigma_t^2}.
\end{align}
Subtracting gives
\begin{align}
\nabla_{\bz_{t-1}} \log P_F^\theta - \nabla_{\bz_{t-1}} \log P_F^{\rm pre}
&=
\frac{\mu_\theta(\bz_t, t,y) - \mu^{\rm pre}(\bz_t, t,y)}{\sigma_t^2},
\end{align}
which is exactly $\nabla_{\bz_{t-1}} \log \tilde{P}_F^\theta(\bz_{t-1} \mid \bz_t,y)$ by definition of the residual transition ratio.
\end{proof}

\textbf{Interpretation.}
Eq.~\eqref{eq:loss_forward} uses exactly this forward-vs-pretrained residual score.
This is a local score-matching identity for the transition ratio; it does not by itself prove the global telescoping statement of Theorem~\ref{thm:distributional}.

\subsection{Reverse Gaussian Score}
\label{app:reverse_gaussian_score}

\begin{proposition}[Gaussian reverse transition score]
\label{prop:reverse_score}
For a Gaussian denoising transition
\[
P_F^\theta(\bz_{t-1} \mid \bz_t,y)
=
\mathcal{N}(\mu_\theta(\bz_t,t,y), \sigma_t^2 \mathbf I),
\]
with variance independent of $\bz_t$, the score with respect to the conditioning latent is
\begin{align}
\nabla_{\bz_t}\log P_F^\theta(\bz_{t-1} \mid \bz_t,y)
&=
\frac{1}{\sigma_t^2}J_{\mu_\theta}(\bz_t,t,y)^\top(\bz_{t-1} - \mu_\theta(\bz_t,t,y)).
\end{align}
\end{proposition}

\begin{proof}
Writing
\[
\log P_F^\theta(\bz_{t-1} | \bz_t,y)
=
-\frac{1}{2\sigma_t^2}\|\bz_{t-1} - \mu_\theta(\bz_t,t,y)\|^2 + C_t,
\]
where $C_t$ does not depend on $\bz_t$, the chain rule gives
\begin{align}
\nabla_{\bz_t}\log P_F^\theta(\bz_{t-1} | \bz_t,y)
&=
-\frac{1}{2\sigma_t^2}
\nabla_{\bz_t}
\left[(\bz_{t-1} - \mu_\theta(\bz_t,t,y))^\top(\bz_{t-1} - \mu_\theta(\bz_t,t,y))\right] \\
&=
\frac{1}{\sigma_t^2}J_{\mu_\theta}(\bz_t,t,y)^\top(\bz_{t-1} - \mu_\theta(\bz_t,t,y)).
\end{align}
\end{proof}

This is the expression used in the implementation of the reverse residual score.

\subsection{Adapting \texorpdfstring{$\nabla$}{nabla}-DB for Forget Energies}
\label{app:forget_energy_adaptation}

For every prompt $y$, the terminal conditional density ratio $\exp(-\widetilde{E}_{\mathcal{C}}(\bx))$ is always positive, so the standard GFlowNet requirement that terminal weights be nonnegative is automatically satisfied.

\begin{proposition}[Derivative of the thresholded forget energy]
\label{prop:threshold_grad}
Let
\begin{align}
\widetilde{E}_{\mathcal{C}}(\bx)
&=
\lambda_{\text{scale}} \cdot
\big(\exp\!\big(\alpha \cdot \max(s_{\mathcal{C}}(\bx) - \tau,\; 0)\big) - 1\big).
\end{align}
Then for $s_{\mathcal{C}}(\bx) \neq \tau$,
\begin{align}
\frac{\partial \widetilde{E}_{\mathcal{C}}}{\partial s_{\mathcal{C}}}
=
\begin{cases}
\lambda_{\text{scale}} \cdot \alpha \cdot \exp(\alpha(s_{\mathcal{C}} - \tau)), & \text{if } s_{\mathcal{C}} > \tau, \\
0, & \text{if } s_{\mathcal{C}} < \tau.
\end{cases}
\end{align}
At the kink $s_{\mathcal{C}}(\bx)=\tau$, $\widetilde{E}_{\mathcal{C}}$ is not differentiable; the implementation uses the zero subgradient.
\end{proposition}

\begin{proof}
If $s_{\mathcal{C}}(\bx) < \tau$, then the maximum term is zero, so $\widetilde{E}_{\mathcal{C}}(\bx)=0$ and the derivative is zero.
If $s_{\mathcal{C}}(\bx) > \tau$, then
\begin{align}
\widetilde{E}_{\mathcal{C}}(\bx)
&=
\lambda_{\text{scale}} \cdot \left(\exp\!\big(\alpha(s_{\mathcal{C}}(\bx)-\tau)\big) - 1\right),
\end{align}
whose derivative is the stated exponential term.
At $s_{\mathcal{C}}(\bx)=\tau$, the left derivative is $0$ while the right derivative is $\lambda_{\text{scale}}\alpha$, so only a subgradient is defined.
\end{proof}

For the thresholded energy in Eq.~\eqref{eq:reward}, the energy gradient vanishes below the threshold and grows exponentially above it:
\begin{itemize}[leftmargin=1.6em,itemsep=0.35em,topsep=0.45em]
    \item The threshold $\tau$ ensures benign images receive zero local energy gradient.
    \item The exponential penalty concentrates updates on clearly concept-resembling images.
    \item The positive terminal weight remains compatible with proportional sampling through $\exp(-\widetilde{E}_{\mathcal{C}}(\bx))$.
\end{itemize}

\subsection{Optimization Caveat}
\label{app:optimization_caveat}

The training objective is a nonconvex score-matching surrogate over the LoRA parameters and the auxiliary residual parameterization.
We therefore do not claim a general ``if and only if'' theorem from $\mathcal{L}=0$ to distributional correctness, nor a global gradient-descent convergence rate for the LoRA-parameterized model.
The rigorous statement is Theorem~\ref{thm:distributional}: if the learned residual transitions satisfy the integrated ratio-form DB relation together with the data-end boundary condition, then the terminal conditional marginal obeys $p_\theta(\bx \mid y) \propto p_{\mathrm{pre}}(\bx \mid y)\exp(-E(\bx))$ for the chosen terminal energy $E$.
In practice, $\mathcal{L}_{\text{fwd}}$, $\mathcal{L}_{\text{rev}}$, and $\mathcal{L}_{\text{term}}$ are the empirical surrogates used to encourage this condition.

\clearpage

\section{Extended Derivations}
\label{app:derivations}

\subsection{Closed-Form Conditional Gibbs Projection}
\label{app:gibbs}

The main text uses a prompt-conditioned target, with fixed prompt population $\pi(y)$.
The optimization is therefore over conditional distributions $p(\cdot \mid y)$, while $\pi(y)$ is held fixed.

\begin{proposition}[Conditional Gibbs projection for concept unlearning]
\label{prop:gibbs}
Consider the constrained problem
\begin{align}
p^*_{\mathcal{C}}
&=
\argmin_{p}
\E_{y\sim\pi}\!\left[
\KL\!\left(p(\cdot \mid y)\,\|\,p_{\mathrm{pre}}(\cdot \mid y)\right)
\right] \\
\text{s.t.}\qquad
\E_{y\sim\pi,\,\bx\sim p(\cdot \mid y)}
\!\left[E_{\mathcal{C}}(\bx)\right]
&\le \delta_{\mathcal{C}},
\end{align}
over conditional distributions $p(\cdot \mid y)$ that are absolutely continuous with respect to $p_{\mathrm{pre}}(\cdot \mid y)$ for $\pi$-almost every $y$ and have finite prompt-averaged energy.
Assume the feasible set is nonempty and contains a strictly feasible conditional distribution.
Then the minimizer is unique up to $\pi$-null prompt sets and has the form
\begin{align}
p^*_{\mathcal{C}}(\bx \mid y)
&=
\frac{p_{\mathrm{pre}}(\bx \mid y)\exp(-\beta E_{\mathcal{C}}(\bx))}
{Z_{\beta,\mathcal{C}}(y)},
\qquad
Z_{\beta,\mathcal{C}}(y)
\coloneqq
\int p_{\mathrm{pre}}(\bx \mid y)\exp(-\beta E_{\mathcal{C}}(\bx))\,d\bx,
\end{align}
for some shared $\beta \ge 0$ with $Z_{\beta,\mathcal{C}}(y) < \infty$ for $\pi$-almost every $y$.
Moreover, $\beta$ satisfies the KKT conditions
\begin{align}
\beta \ge 0,\qquad
\E_{y\sim\pi,\,\bx\sim p^*_{\mathcal{C}}(\cdot \mid y)}
\!\left[E_{\mathcal{C}}(\bx)\right]
\le \delta_{\mathcal{C}},\qquad
\beta\!\left(
\E_{y\sim\pi,\,\bx\sim p^*_{\mathcal{C}}(\cdot \mid y)}
\!\left[E_{\mathcal{C}}(\bx)\right]
- \delta_{\mathcal{C}}
\right)=0.
\end{align}
\end{proposition}

\begin{proof}
The feasible set is convex, and the prompt-averaged conditional KL is strictly convex in $p(\cdot \mid y)$ for $\pi$-almost every $y$, so any optimizer is unique up to $\pi$-null prompt sets.
Because a strictly feasible point exists, Slater's condition holds and the KKT conditions are sufficient.
Form the Lagrangian with shared multiplier $\beta \ge 0$ for the averaged forgetting constraint and a prompt-dependent normalization multiplier $\lambda(y)$:
\begin{align}
\mathcal{L}(p, \beta, \lambda(\cdot))
&=
\E_{y\sim\pi}\!\left[
\int p(\bx \mid y)\log\frac{p(\bx \mid y)}{p_{\mathrm{pre}}(\bx \mid y)}\,d\bx
\right] \nonumber\\
&\quad
+ \beta\!\left(
\E_{y\sim\pi}\!\left[
\int p(\bx \mid y)E_{\mathcal{C}}(\bx)\,d\bx
\right]
-\delta_{\mathcal{C}}
\right) \nonumber\\
&\quad
+ \E_{y\sim\pi}\!\left[
\lambda(y)\!\left(\int p(\bx \mid y)\,d\bx - 1\right)
\right].
\end{align}
Setting the first variation with respect to $p(\bx \mid y)$ to zero for each prompt gives
\begin{align}
\log\frac{p(\bx \mid y)}{p_{\mathrm{pre}}(\bx \mid y)}
+ 1 + \beta E_{\mathcal{C}}(\bx) + \lambda(y)
&=0,
\end{align}
hence
\begin{align}
p(\bx \mid y)
&=
p_{\mathrm{pre}}(\bx \mid y)
\exp\!\left(-1-\lambda(y)-\beta E_{\mathcal{C}}(\bx)\right).
\end{align}
Normalizing for each prompt yields
\begin{align}
\exp(1+\lambda(y))
&=
\int p_{\mathrm{pre}}(\bx \mid y)\exp(-\beta E_{\mathcal{C}}(\bx))\,d\bx
\eqqcolon
Z_{\beta,\mathcal{C}}(y),
\end{align}
so
\begin{align}
p^*_{\mathcal{C}}(\bx \mid y)
&=
\frac{p_{\mathrm{pre}}(\bx \mid y)\exp(-\beta E_{\mathcal{C}}(\bx))}
{Z_{\beta,\mathcal{C}}(y)}.
\end{align}
The remaining KKT conditions are precisely the shared-multiplier conditions for the prompt-averaged constraint:
\begin{align}
\beta \ge 0,\qquad
\E_{y\sim\pi,\,\bx\sim p^*_{\mathcal{C}}(\cdot \mid y)}
\!\left[E_{\mathcal{C}}(\bx)\right]
\le \delta_{\mathcal{C}},\qquad
\beta\!\left(
\E_{y\sim\pi,\,\bx\sim p^*_{\mathcal{C}}(\cdot \mid y)}
\!\left[E_{\mathcal{C}}(\bx)\right]
-\delta_{\mathcal{C}}
\right)=0.
\end{align}
\end{proof}

Thus $\beta$ is tied to the active constraint level in the stated constrained problem.
When $\beta$ is instead chosen directly as a hyperparameter, one is solving the relaxed penalized objective rather than fixing $\delta$ in advance.
In the thresholded practical surrogate used in the main method, this scalar strength is absorbed into $\widetilde{E}_{\mathcal{C}}$ through $\lambda_{\text{scale}}$, while the target remains prompt-conditioned through $p_{\mathrm{pre}}(\bx \mid y)$ and $Z_{\beta,\mathcal{C}}(y)$.

\subsection{Gaussian Instantiation of \texorpdfstring{$\nabla$}{nabla}-DB}
\label{app:gaussian}

For Gaussian denoising transitions $P_F^\theta(\bz_{t-1} \mid \bz_t,y) = \mathcal{N}(\mu_\theta(\bz_t, t,y), \sigma_t^2 \mathbf{I})$ and $P_F^{\rm pre}(\bz_{t-1} \mid \bz_t,y) = \mathcal{N}(\mu^{\rm pre}(\bz_t, t,y), \sigma_t^2 \mathbf{I})$:
\begin{align}
\nabla_{\bz_{t-1}} \log P_F^\theta &= -\frac{\bz_{t-1} - \mu_\theta(\bz_t, t,y)}{\sigma_t^2}, \\
\nabla_{\bz_{t-1}} \log P_F^{\rm pre} &= -\frac{\bz_{t-1} - \mu^{\rm pre}(\bz_t, t,y)}{\sigma_t^2}.
\end{align}

In the DDPM framework with matched variances, the residual forward score simplifies to
\begin{align}
\nabla_{\bz_{t-1}} \log \tilde{P}_F^\theta(\bz_{t-1} \mid \bz_t,y) &= \frac{\mu_\theta(\bz_t, t,y) - \mu^{\rm pre}(\bz_t, t,y)}{\sigma_t^2}.
\end{align}

This is the finetuned-vs-pretrained model deviation divided by the noise variance.
In our implementation, this is computed as \texttt{(m\_theta - m\_ref) / sigma\_sq\_t}.

\subsection{Non-Ideal Model Correction}
\label{app:correction}

The pretrained model $P_F^{\rm pre}$ was trained with a denoising objective rather than a GFlowNet objective, so it only approximately satisfies detailed balance.
The residual term $g_\phi$ parameterizes this discrepancy within the continuation-score approximation.
In our implementation, $g_\phi$ is prompt-independent; prompt dependence enters through the denoiser-based clean estimate $\hat{\bx}_\theta(\bz_t,t,y)$.

Let $\Delta_t$ be the pretrained model's approximation error at timestep $t$.
The forward-looking parameterization becomes
\begin{align}
\nabla \log \tilde{F}(\bz_t,y) &= -\nabla_{\bz_t} \widetilde{E}_{\mathcal{C}}(\hat{\bx}_\theta(\bz_t,t,y)) \cdot \bar\alpha_t + \underbrace{g_\phi(\bz_t,t)}_{\approx \Delta_t + \text{higher-order terms}}.
\end{align}

The terminal loss $\|g_\phi(\bz_0)\|^2$ ensures this correction vanishes at the data end, so the boundary flow is set by the chosen terminal energy rather than by an unconstrained correction term.

\clearpage

\section{Additional Experimental Details}
\label{app:exp_details}
\label{app:pipeline}

This section reports the additional implementation details required to reproduce the experiments in Section~\ref{sec:experiments}: the encoders and base diffusion model, the prompt sets used for forgetting and benign retention, the compute budget, and the practical guidance for tuning the small set of method-specific hyperparameters.

\subsection{CLIP Model}
\label{app:clip_model}

We use \texttt{openai/clip-vit-large-patch14} (ViT-L/14, $224\times 224$ input, output dimension 768) for both the forget-energy computation and the evaluation metrics. A single CLIP model is shared across (i) the descriptor set $\mathcal{Q}_{\mathcal{C}}$ used to define the forget energy, and (ii) the evaluation similarity scores reported in the main text. Sharing the encoder ensures that the forget signal during training and the evaluation metric at test time are measured in the same embedding space; a different CLIP variant could be substituted, with the threshold $\tau$ recalibrated, without changing the rest of the pipeline.

\subsection{Base Diffusion Model and Adapters}
\label{app:base_model}

All experiments adapt Stable Diffusion v1.5 (\texttt{runwayml/stable-diffusion-v1-5}). LoRA adapters are inserted into the U-Net cross- and self-attention projections $(W_Q,W_K,W_V,W_O)$ at rank $r=8$, with the standard $\alpha=r$ scaling. Only the LoRA weights $\Delta\theta$ and the auxiliary residual parameterization $g_\phi$ (channels $\{64,128,256,256\}$, time-embedded) receive gradient updates; the U-Net backbone, latent autoencoder, text encoder, and CLIP model are all frozen. We use 50-step DDPM sampling at training and inference time.

\subsection{Prompt Sets and Generation Strategy}
\label{app:prompt_sets}

\paragraph{Why LLM-based prompt generation.}
Prompt construction for concept unlearning is a well-established design choice in the literature, and existing works fall broadly into two camps: \emph{template-based} prompts (a small set of fixed scaffolds with concept-name slots) and \emph{LLM-based} prompts (generated by a language model under task-specific instructions). Both are used widely; we adopt the LLM-based route because it produces a more varied, natural, and semantically diverse prompt set than templates. An LLM can compose grammatically natural phrasings, vary scene context, framing, lighting, and activity, and cover a much broader region of plausible prompt space without manual enumeration. This matters for unlearning evaluation because narrow prompt templates can mask retention failures that only surface under naturalistic prompt distributions. We use GPT-4o-mini, which provides this diversity at a low API cost while still respecting the formatting constraints required for clean diffusion prompts.

\paragraph{Prompt generation strategy.}
Prompts are generated under carefully designed system prompts that enforce three properties: (i) \textbf{visual grounding}---the concept must be a clearly visible, concrete subject that the diffusion model can depict in-frame, with the physical sense chosen for polysemous concepts (e.g., ``apple'' = fruit, not the company); (ii) \textbf{scene and category diversity}---prompts vary across indoor/outdoor settings, day/night, weather, activities, solitary vs.\ small-group scenes, and near vs.\ far framings, with no near-duplicate phrasings; and (iii) \textbf{lexical simplicity}---4--10 words, present tense, $\le 1$ adjective, no special tags, parameters, or stylistic suffixes. The same pipeline is reused for forget and retain prompts, with minor wording changes in the system prompt to switch the target subject and to exclude the unlearning concept token from the retain prompts.

\begin{tcolorbox}[colback=gray!5!white, colframe=gray!75!black, title={\textbf{Forget Prompt Generation -- System Prompt}}, fonttitle=\bfseries, breakable]
{\ttfamily\footnotesize
You are a dataset prompt writer for diffusion models (e.g., Stable Diffusion). Your job is to produce short, simple, natural-language prompts that depict a clearly visible instance of the specified concept in-frame.

GENERAL STYLE: 4--10 words per prompt; one simple main clause; present tense preferred; everyday vocabulary only; 0--1 adjectives; no tags/parameters (no colons, aspect ratios, seeds, CFG, hashtags); minimal punctuation.

VISUAL GROUNDING (MANDATORY): The concept MUST be a visible, concrete subject in the image. If the concept is polysemous, use the physical object sense. The scene must make the concept clearly depictable and in-frame. Match number and determiners consistently (a/an/one/two/many).

DIVERSITY REQUIREMENTS: Across prompts, balance indoor vs.\ outdoor, day vs.\ night, weather conditions, everyday activities, solitary vs.\ small-group scenes, and near vs.\ far framing. Avoid near-duplicates and repeated sentence structure.
}
\end{tcolorbox}

\paragraph{Retain prompt generation.}
For the optional retain-prompt augmentation reported in Section~\ref{sec:results} (Q4), we use a two-stage pipeline. We first ask the LLM to enumerate $\sim$50 candidate related concepts for each unlearning target, covering close siblings in the same category, broader category/hypernym terms, contextually co-occurring concepts, and analogous concepts from other domains, while explicitly excluding synonyms, subtypes, and parts of the target. From this pool we narrow to a smaller neighbor set using CLIP text--text similarity (or denoiser-based score-prediction distance as an alternative). For each selected neighbor, the LLM generates 10 diverse scene prompts under a replay-style system prompt that constrains the output to exclude the unlearning concept token and any obvious synonym.

\begin{tcolorbox}[colback=gray!5!white, colframe=gray!75!black, title={\textbf{Retain Prompt Generation -- System Prompt}}, fonttitle=\bfseries, breakable]
{\ttfamily\footnotesize
You are a prompt generator for Stable Diffusion models.

INPUT: UNLEARN\_CONCEPT (the concept being unlearned; must NOT appear in any prompt) and RELATED\_CONCEPT (the target subject for retention).

CONSTRAINTS: 5--10 words per prompt; present tense; simple, everyday vocabulary; $\le 1$ comma and $\le 1$ adjective per prompt; no special tags or parameters; do not include UNLEARN\_CONCEPT or any obvious synonym/nickname.

CONTENT: Describe simple, plausible scenes involving RELATED\_CONCEPT. Ensure variety: mix indoor/outdoor settings, day/night, weather, activities, solitary vs.\ group, near vs.\ far framing. Do not repeat sentence structure; avoid near-duplicates. Keep all prompts safe and neutral in tone.

OUTPUT: Exactly 10 prompts, one per line. No numbering, bullet points, section titles, or concept labels.
}
\end{tcolorbox}

\paragraph{Prompt sets used in experiments.}
We use four prompt sets per concept: a forget set used during training, a held-out forget set used at evaluation, a benign retention set covering nearby but distinct concepts, and a general benign set covering unrelated content. Sizes and representative phrasings are summarized below.
\begin{itemize}[leftmargin=1.6em,itemsep=0.35em,topsep=0.45em]
    \item \textbf{Van~Gogh forget (training):} 50 prompts of the form ``a painting in the style of Van~Gogh'', ``Van~Gogh starry night'', ``Van~Gogh self-portrait'', etc. These are the prompts on which residual $\nabla$-DB losses are computed.
    \item \textbf{Van~Gogh forget (held-out):} a disjoint set of 50 forget prompts used only for evaluation, ensuring that the reported forget rate is not driven by training-prompt memorization.
    \item \textbf{Van~Gogh benign:} 200 prompts spanning landscapes, portraits, animals, and objects in non-Van~Gogh styles, used to measure benign-concept retention.
    \item \textbf{Pikachu forget (training/held-out):} 50 + 50 prompts of the form ``pikachu'', ``a picture of pikachu'', ``pikachu in a forest'', etc.
    \item \textbf{Pikachu benign:} 200 prompts for non-Pikachu characters and general objects.
\end{itemize}
The same retention prompts are used across methods, so retention numbers are directly comparable.

\subsection{Optimization and Schedule}
\label{app:optimization}

LoRA parameters and $\phi$ are jointly optimized with AdamW at learning rate $1\!\times\!10^{-4}$, $(\beta_1,\beta_2)=(0.9,0.999)$, weight decay $10^{-2}$, and gradient clipping at norm $1.0$. We use a batch size of 4 forget prompts, with one denoising trajectory per prompt and $|\mathcal{T}|=0.1T=5$ score-matched timesteps per trajectory. Training runs for 1{,}000 batches per concept, which is sufficient for convergence in all experiments. The loss weights in Eq.~\eqref{eq:loss_total} are $\lambda_{\rm rev}=1.0$ and $\lambda_{\rm term}=0.1$.

\subsection{Compute}
\label{app:compute}

All experiments are run on a single NVIDIA A100 (80\,GB). End-to-end training takes 3--5 hours per concept, dominated by the trajectory rollout (Phase 1 of Algorithm~\ref{alg:training}) and the Jacobian-vector product in the reverse residual score (Phase 3). Inference uses standard 50-step DDPM sampling and incurs no additional cost beyond the base diffusion model -- the LoRA correction is folded into the U-Net forward pass.

\subsection{Hyperparameter Sensitivity}
\label{app:hyperparameter_sensitivity}

The method exposes three concept-level hyperparameters: the energy scale $\lambda_{\text{scale}}$, the slope $\alpha$ inside the exponential, and the CLIP-similarity threshold $\tau$. Their roles, recommended ranges, and tuning strategy are as follows.
\begin{itemize}[leftmargin=1.6em,itemsep=0.35em,topsep=0.45em]
    \item \textbf{$\lambda_{\text{scale}}$ (tilt magnitude).} Most sensitive overall: it controls the magnitude of the terminal energy and hence the strength of the unlearning tilt. Too small leaves the concept partially intact; too large destabilizes training. We sweep $\lambda_{\text{scale}}\in\{0.5,1,2,5\}$ and pick the smallest value that achieves a stable forget rate.
    \item \textbf{$\tau$ (decision boundary).} Defines the no-tilt benign region. Set from the empirical CLIP-similarity distribution of the pretrained model on benign prompts: typically the $90$--$95$ percentile of benign similarities, so that benign images receive zero terminal-energy gradient. In our experiments, we keep $\tau$ fixed across concepts within each category, using $\tau=0.21$ for Objects, $\tau=0.22$ for Characters, and $\tau=0.20$ for Style.
    \item \textbf{$\alpha$ (slope).} Controls how sharply the energy grows above $\tau$. Less sensitive than $\lambda_{\text{scale}}$ and $\tau$; we use $\alpha=10$ throughout.
\end{itemize}
\textbf{Recommended tuning order.} (1) Estimate $\tau$ from the benign-prompt CLIP-similarity histogram. (2) Tune $\lambda_{\text{scale}}$ for stable training and adequate forget rate. (3) If forget images are slipping just above $\tau$ without being pushed away, increase $\alpha$ before increasing $\lambda_{\text{scale}}$ further.

\subsection{Evaluation Protocol}
\label{app:eval_protocol}

This section expands the evaluation summary in Section~\ref{sec:metrics} with the full automated evaluation methodology, dataset construction, and metric definitions used to produce the main quantitative results.

\paragraph{VLM-based automated evaluation.}
To automate and standardize the assessment of unlearning and retention, we use \textbf{Qwen2.5-VL-7B-Instruct}~\citep{bai2025qwen2} as an impartial evaluator. Modern VLMs are well suited to this role: they are explicitly trained for open-vocabulary visual question answering, and for binary, closed-ended questions of the form \textit{``Does this image contain $\langle$concept$\rangle$?''} they produce robust judgments across a much wider semantic range than dedicated classifiers. For each accuracy metric, we generate a set of images per prompt and aggregate the VLM's per-image yes/no answers into the reported score.

\paragraph{Why not classifiers.}
We deliberately avoid fixed-vocabulary classifiers (e.g., ImageNet-1k pre-trained models) and the category-specific heads provided by benchmarks such as UnlearnCanvas~\citep{zhang2024unlearncanvas} as primary evaluators, for three reasons:
\begin{itemize}[leftmargin=1.6em,itemsep=0.35em,topsep=0.45em]
    \item \textbf{Restricted concept space.} An ImageNet-trained classifier can only score concepts inside its 1000-class vocabulary, so it cannot evaluate styles (e.g., \textit{Van~Gogh}, \textit{Monet}, cartoon), specific identities (\textit{Brad~Pitt}, \textit{Lionel~Messi}), fictional characters (\textit{Pikachu}, \textit{Mickey~Mouse}), or many fine-grained variations that real unlearning evaluations need. Restricting the evaluator restricts the experimental setup to whatever the classifier happens to know, which is exactly the wrong constraint for studying broad concept unlearning.
    \item \textbf{UnlearnCanvas-specific limitations.} The classifiers shipped with UnlearnCanvas similarly cover mostly objects and styles from its own dataset, and the artistic styles in UnlearnCanvas are visually quite different from the canonical referent of those style names (e.g., the dataset's ``Van~Gogh''-tagged images differ in color palette, brushwork, and composition from the historical paintings); we observed that the provided style classifier does not generalize well to images that follow the canonical style rather than the UnlearnCanvas style. This makes the classifier's verdict unreliable as a stand-alone concept-unlearning metric.
    \item \textbf{Poor performance on AI-generated images.} Standard classifiers are trained on natural-image distributions; they degrade noticeably on diffusion-generated images, where backgrounds, lighting, and scene composition can be unnaturally smooth, stylized, or artifact-prone. These shifts hurt classifier accuracy even when the depicted concept is clearly recognizable to a human or VLM, biasing the resulting unlearning/retention numbers in opaque ways.
\end{itemize}
We empirically verified these failure modes on out-of-distribution and AI-generated images from the original Stable Diffusion model: conventional classifiers produced inconsistent and often incorrect labels, whereas Qwen2.5-VL produced substantially more accurate yes/no judgments aligned with manual inspection. Using a VLM also lifts the cap on which concepts we can evaluate---styles, characters, identities, and arbitrary objects can all be queried with the same evaluator---so the experimental setup is no longer bounded by a particular classifier's vocabulary.

\paragraph{CLIP score as a complementary metric.}
Alongside VLM-based accuracy, we report \textbf{CLIP score} ($U_{\text{clip}}$, $RR_{\text{clip}}$, $GR_{\text{clip}}$) for every prompt set. CLIP score is the standard text--image alignment metric in the diffusion-model unlearning literature and is reported by essentially all existing baselines, so including it ensures our evaluation is directly comparable to prior work and not solely dependent on the VLM judgment. In our setup, the VLM-based accuracy captures whether the concept is present or absent (the discrete ``did it forget?'' question), while CLIP score captures graded text--image alignment on the benign part of the prompt. The two metrics are therefore complementary: a high VLM accuracy with a high CLIP score indicates that the target concept is removed without damaging the surrounding scene description.

\paragraph{Evaluation datasets and scale.}
We construct three disjoint evaluation prompt sets per checkpoint: the forget evaluation set $\mathcal{D}_{\text{eval\_forget}}$, the related-retain set $\mathcal{D}_{\text{eval\_related}}$, and the general-retain set $\mathcal{D}_{\text{eval\_general}}$. All evaluation prompts are held out from the training prompts to test generalization rather than memorization. We track approximately \textbf{60 distinct concepts} (spanning unlearning targets, related concepts, and general themes), with \textbf{20 unique evaluation prompts per concept} and \textbf{8 images per prompt}, yielding approximately \textbf{10{,}000 images per model checkpoint}. The general-retain set is fixed and unrelated to any unlearning target, providing a control group for measuring catastrophic forgetting on the model's base capabilities.

\paragraph{Related concept identification.}
To identify semantically entangled concepts for the related-retain set, we follow the methodology of EraseBench~\citep{amara2025erasebench}. We use GPT-4o-mini with the system prompt: \textit{``Your main task is to help identify concepts for evaluating text-to-image models. The key idea is to identify 3--4 concepts that are semantically entangled with the \textbf{Given Concept}\dots''}. This yields adjacent concepts (e.g., \textit{Pikachu} $\rightarrow$ \textit{Squirtle, Charmander}; \textit{Van Gogh} $\rightarrow$ \textit{Impressionism}) that are used to measure the ripple effect~\citep{amara2025erasebench,zhang2024unlearncanvas}.

\paragraph{Metric definitions.}
For each concept, we report:
\begin{itemize}[leftmargin=1.6em,itemsep=0.35em,topsep=0.45em]
    \item \textbf{Unlearning accuracy ($U_{\text{acc}}$) and CLIP score ($U_{\text{clip}}$).} $U_{\text{acc}}$ measures erasure efficacy; high $U_{\text{acc}}$ indicates successful unlearning. $U_{\text{clip}}$ serves as a proxy for ``in-prompt retainability''~\citep{ren2025sixcd}, ensuring that while the target concept is removed, the model still respects the benign context of the prompt. The ideal outcome is a \textbf{high $U_{\text{acc}}$} (concept removed) paired with a \textbf{high $U_{\text{clip}}$} (context preserved).
    \item \textbf{Related retention ($RR_{\text{acc}}$) and CLIP score ($RR_{\text{clip}}$).} To quantify the ripple effect, we evaluate performance on the semantically entangled concepts identified above. A high $RR_{\text{acc}}$ indicates surgical unlearning with minimal collateral damage. $RR_{\text{clip}}$ ensures that visual quality and text alignment for these neighboring concepts remain intact.
    \item \textbf{General retention ($GR_{\text{acc}}$) and CLIP score ($GR_{\text{clip}}$).} $GR_{\text{acc}}$ assesses stability on a broad set of unrelated concepts, detecting catastrophic forgetting; $GR_{\text{clip}}$ verifies that general instruction-following and text-to-image alignment remain intact after unlearning.
\end{itemize}

\paragraph{Generative quality.}
Beyond accuracy, we evaluate the fundamental generative quality of the unlearned model following standard practice. We adopt the \textbf{Fr\'echet Inception Distance (FID)}~\citep{heusel2017fid}, computed using the evaluation script and methodology used in the UnlearnCanvas benchmark~\citep{zhang2024unlearncanvas}. To measure preservation of distributional quality, we compute FID between images generated by the unlearned model and images generated by the original Stable Diffusion model on the same evaluation prompts, treating the pretrained-model distribution as the reference for benign generation. For all CLIP-score calculations ($U_{\text{clip}}$, $RR_{\text{clip}}$, $GR_{\text{clip}}$), we use \textbf{\texttt{openai/clip-vit-large-patch14}}, ensuring alignment with standard evaluation protocols.

\paragraph{FADE computation and retain-only reference.}
\textbf{FADE}~\citep{cho2025fade} is computed as defined in Eq.~\eqref{eq:fade}; for diffusion models it is tractably estimated via weighted differences of denoising MSE losses between the unlearned model and a retain-only reference. Constructing a meaningful retain-only reference is the main practical challenge of FADE evaluation, since it requires a model that has \textit{never seen} the target concept yet otherwise matches the pretrained-and-finetuned distribution.

We follow the UnlearnCanvas~\citep{zhang2024unlearncanvas} protocol to obtain such a reference. The pipeline has two stages:
\begin{enumerate}[leftmargin=1.6em,itemsep=0.35em,topsep=0.45em]
    \item \textbf{Pretrained-and-finetuned model (used as the unlearning starting point).} Starting from Stable Diffusion v1.5, we finetune on the \emph{full} UnlearnCanvas dataset, which covers all styles. Concept unlearning is performed on top of this checkpoint, so all forgetting and retention metrics are measured relative to a model that has been explicitly trained on the target style.
    \item \textbf{Retain-only reference for FADE.} For each target style $\mathcal{C}_{\mathrm{style}}$, we separately finetune a fresh Stable Diffusion v1.5 on UnlearnCanvas with $\mathcal{C}_{\mathrm{style}}$ removed (i.e., on UnlearnCanvas$\setminus\mathcal{C}_{\mathrm{style}}$). This produces a model whose training distribution matches the pretrained-and-finetuned model in everything except the omitted style, exactly the gold-standard retain-only reference required by FADE.
\end{enumerate}

We carry out this construction for three artistic-style concepts that are well-supported by the UnlearnCanvas dataset---\textbf{Van~Gogh style}, \textbf{Picasso style}, and \textbf{Monet style}---and compute the per-style FADE between the unlearned model and the corresponding retain-only reference. The \textbf{Avg.\ FADE} value reported in Table~\ref{tab:main_comparison} is the mean of these three per-style FADE scores. This averaging matches the granularity at which the retain-only references are practical to train, while still providing a principled distributional comparison against a genuine concept-free baseline rather than only the pretrained model.

\clearpage

\section{Performance Across Multiple Diffusion Architectures}
\label{app:architecture_results}

We additionally evaluate Pikachu concept unlearning across multiple base diffusion architectures using the same evaluation protocol.
The forget-prompt classifier accuracy column is reported directly, so lower values indicate stronger erasure; retain accuracies and CLIP scores remain higher-is-better.

\begin{table}[htbp]
\centering
\caption{Pikachu unlearning performance across base diffusion architectures.}
\label{tab:architecture_pikachu}
\vspace{0.3em}
\setlength{\tabcolsep}{4pt}
\resizebox{\textwidth}{!}{%
\begin{tabular}{@{}lcccccc@{}}
\toprule
\textbf{Base model} & $F_{\mathrm{Acc}}\downarrow$ & $U_{\mathrm{CLIP}}\uparrow$ & $RR_{\mathrm{Acc}}\uparrow$ & $RR_{\mathrm{CLIP}}\uparrow$ & $GR_{\mathrm{Acc}}\uparrow$ & $GR_{\mathrm{CLIP}}\uparrow$ \\
\midrule
SD v1.4 & 0.01 & 28.0 & 0.75 & 30.9 & 0.84 & 31.4 \\
SD v1.5 & 0.00 & 29.9 & 0.76 & 33.0 & 0.88 & 32.1 \\
SD v2 & 0.01 & 29.9 & 0.73 & 32.9 & 0.83 & 32.0 \\
SD3 & 0.01 & 29.8 & 0.73 & 32.9 & 0.86 & 32.2 \\
\bottomrule
\end{tabular}%
}
\end{table}

\clearpage

\section{Per-Concept Quantitative Results}
\label{app:per_concept_results}

This section provides the full per-concept breakdown of the quantitative results summarised in Table~\ref{tab:main_comparison}.
Each table reports results for a single unlearning target across all methods, including unlearning accuracy ($U_{\mathrm{Acc}}\uparrow$), related-retain accuracy ($RR_{\mathrm{Acc}}\uparrow$), general-retain accuracy ($GR_{\mathrm{Acc}}\uparrow$), and the corresponding CLIP scores ($\uparrow$).
For methods where no model targeting the concept directly was available, the entry is omitted (--).

\begin{table}[htbp]
\centering
\caption{Concept unlearning results for \textbf{Apple}.}
\label{tab:concept_apple}
\vspace{0.3em}
\setlength{\tabcolsep}{4pt}
\resizebox{\textwidth}{!}{%
\begin{tabular}{@{}lcccccc@{}}
\toprule
\textbf{Method} & $U_{\mathrm{Acc}}\uparrow$ & $U_{\mathrm{CLIP}}\uparrow$ & $RR_{\mathrm{Acc}}\uparrow$ & $RR_{\mathrm{CLIP}}\uparrow$ & $GR_{\mathrm{Acc}}\uparrow$ & $GR_{\mathrm{CLIP}}\uparrow$ \\
\midrule
ESD-u & \cellcolor{tabthird}0.88 & 28.7 & 0.66 & 29.4 & 0.78 & 31.7 \\
ESD-x & 0.82 & 29.4 & 0.69 & 29.5 & 0.80 & 31.9 \\
UCE & 0.75 & 29.4 & \cellcolor{tabsecond}0.91 & \cellcolor{tabthird}29.9 & 0.81 & 31.3 \\
CA & 0.69 & 29.6 & \cellcolor{tabfirst}0.97 & \cellcolor{tabsecond}30.7 & \cellcolor{tabfirst}0.90 & \cellcolor{tabfirst}32.8 \\
MACE & 0.31 & \cellcolor{tabfirst}31.4 & 0.84 & 29.6 & 0.80 & 31.9 \\
DUO & \cellcolor{tabfirst}0.98 & 29.4 & \cellcolor{tabthird}0.85 & \cellcolor{tabsecond}30.7 & \cellcolor{tabsecond}0.89 & \cellcolor{tabsecond}32.7 \\
EraseFlow & 0.76 & \cellcolor{tabsecond}30.2 & 0.72 & 29.1 & 0.82 & 32.3 \\
Meta & \cellcolor{tabfirst}0.98 & 25.0 & 0.42 & 26.6 & 0.72 & 30.6 \\
SHS & 0.76 & 25.0 & 0.56 & 26.1 & 0.53 & 27.8 \\
EDiff & 0.51 & \cellcolor{tabthird}29.9 & 0.82 & \cellcolor{tabfirst}30.9 & \cellcolor{tabthird}0.88 & \cellcolor{tabthird}32.5 \\
\midrule
\textbf{\textsc{TILDE}} & \cellcolor{tabsecond}0.97 & 26.8 & 0.75 & 28.4 & 0.85 & 31.9 \\
\bottomrule
\end{tabular}%
}
\end{table}

\begin{table}[htbp]
\centering
\caption{Concept unlearning results for \textbf{Banana}.}
\label{tab:concept_banana}
\vspace{0.3em}
\setlength{\tabcolsep}{4pt}
\resizebox{\textwidth}{!}{%
\begin{tabular}{@{}lcccccc@{}}
\toprule
\textbf{Method} & $U_{\mathrm{Acc}}\uparrow$ & $U_{\mathrm{CLIP}}\uparrow$ & $RR_{\mathrm{Acc}}\uparrow$ & $RR_{\mathrm{CLIP}}\uparrow$ & $GR_{\mathrm{Acc}}\uparrow$ & $GR_{\mathrm{CLIP}}\uparrow$ \\
\midrule
ESD-u & \cellcolor{tabsecond}0.99 & 27.3 & 0.41 & 29.5 & 0.76 & 31.5 \\
ESD-x & 0.97 & 27.4 & 0.55 & 30.6 & 0.80 & 31.9 \\
UCE & 0.93 & 26.7 & \cellcolor{tabthird}0.88 & \cellcolor{tabthird}31.0 & 0.85 & 31.8 \\
CA & 0.91 & 27.3 & \cellcolor{tabfirst}0.94 & 30.9 & \cellcolor{tabfirst}0.90 & \cellcolor{tabsecond}32.6 \\
MACE & 0.69 & 26.8 & 0.72 & 29.7 & 0.81 & 31.8 \\
DUO & \cellcolor{tabfirst}1.00 & \cellcolor{tabfirst}29.4 & 0.85 & \cellcolor{tabfirst}32.2 & \cellcolor{tabsecond}0.88 & \cellcolor{tabfirst}32.7 \\
EraseFlow & 0.85 & \cellcolor{tabthird}29.1 & 0.72 & 30.3 & 0.83 & 32.2 \\
Meta & 0.92 & 26.1 & 0.33 & 27.8 & 0.73 & 30.6 \\
SHS & 0.82 & 22.9 & 0.48 & 25.9 & 0.49 & 27.4 \\
EDiff & 0.66 & \cellcolor{tabsecond}29.3 & \cellcolor{tabsecond}0.89 & \cellcolor{tabsecond}32.1 & \cellcolor{tabsecond}0.88 & \cellcolor{tabthird}32.5 \\
\midrule
\textbf{\textsc{TILDE}} & \cellcolor{tabthird}0.98 & 22.6 & 0.78 & 29.4 & \cellcolor{tabthird}0.86 & 31.9 \\
\bottomrule
\end{tabular}%
}
\end{table}

\begin{table}[htbp]
\centering
\caption{Concept unlearning results for \textbf{Brad Pitt}.}
\label{tab:concept_brad_pitt}
\vspace{0.3em}
\setlength{\tabcolsep}{4pt}
\resizebox{\textwidth}{!}{%
\begin{tabular}{@{}lcccccc@{}}
\toprule
\textbf{Method} & $U_{\mathrm{Acc}}\uparrow$ & $U_{\mathrm{CLIP}}\uparrow$ & $RR_{\mathrm{Acc}}\uparrow$ & $RR_{\mathrm{CLIP}}\uparrow$ & $GR_{\mathrm{Acc}}\uparrow$ & $GR_{\mathrm{CLIP}}\uparrow$ \\
\midrule
ESD-u & 0.96 & \cellcolor{tabthird}28.3 & 0.41 & 30.2 & 0.79 & 31.4 \\
ESD-x & 0.94 & \cellcolor{tabthird}28.3 & 0.47 & 31.3 & 0.77 & 31.1 \\
UCE & 0.82 & \cellcolor{tabsecond}32.2 & \cellcolor{tabfirst}0.88 & \cellcolor{tabfirst}34.6 & \cellcolor{tabthird}0.87 & \cellcolor{tabfirst}32.5 \\
CA & \cellcolor{tabsecond}0.99 & 26.9 & \cellcolor{tabsecond}0.70 & \cellcolor{tabthird}33.7 & \cellcolor{tabsecond}0.88 & \cellcolor{tabsecond}32.4 \\
MACE & \cellcolor{tabfirst}1.00 & 17.2 & 0.00 & 16.6 & 0.02 & 19.5 \\
DUO & \cellcolor{tabfirst}1.00 & 26.2 & 0.28 & 29.2 & 0.86 & \cellcolor{tabsecond}32.4 \\
EraseFlow & \cellcolor{tabthird}0.98 & 25.5 & 0.06 & 25.1 & 0.58 & 28.9 \\
Meta & \cellcolor{tabthird}0.98 & 25.5 & 0.12 & 26.9 & 0.75 & 31.0 \\
SHS & 0.73 & 26.1 & 0.50 & 27.6 & 0.43 & 26.3 \\
EDiff & 0.39 & \cellcolor{tabfirst}32.4 & \cellcolor{tabfirst}0.88 & \cellcolor{tabsecond}34.0 & \cellcolor{tabfirst}0.89 & \cellcolor{tabsecond}32.4 \\
\midrule
\textbf{\textsc{TILDE}} & \cellcolor{tabthird}0.98 & 27.2 & \cellcolor{tabthird}0.55 & 30.3 & \cellcolor{tabsecond}0.88 & \cellcolor{tabthird}32.2 \\
\bottomrule
\end{tabular}%
}
\end{table}

\begin{table}[htbp]
\centering
\caption{Concept unlearning results for \textbf{Cartoon Style}.}
\label{tab:concept_cartoon_style}
\vspace{0.3em}
\setlength{\tabcolsep}{4pt}
\resizebox{\textwidth}{!}{%
\begin{tabular}{@{}lcccccc@{}}
\toprule
\textbf{Method} & $U_{\mathrm{Acc}}\uparrow$ & $U_{\mathrm{CLIP}}\uparrow$ & $RR_{\mathrm{Acc}}\uparrow$ & $RR_{\mathrm{CLIP}}\uparrow$ & $GR_{\mathrm{Acc}}\uparrow$ & $GR_{\mathrm{CLIP}}\uparrow$ \\
\midrule
ESD-u & 0.70 & 28.5 & \cellcolor{tabthird}0.80 & 30.1 & 0.83 & \cellcolor{tabthird}32.1 \\
ESD-x & 0.89 & 28.4 & \cellcolor{tabthird}0.80 & 30.2 & 0.84 & \cellcolor{tabthird}32.1 \\
UCE & 0.24 & \cellcolor{tabfirst}32.1 & \cellcolor{tabfirst}0.94 & \cellcolor{tabfirst}31.8 & \cellcolor{tabfirst}0.89 & \cellcolor{tabfirst}32.8 \\
CA & \cellcolor{tabsecond}0.97 & 28.4 & 0.79 & \cellcolor{tabthird}30.5 & \cellcolor{tabthird}0.87 & \cellcolor{tabsecond}32.5 \\
MACE & \cellcolor{tabfirst}1.00 & 21.4 & 0.00 & 20.9 & 0.02 & 19.3 \\
DUO & 0.23 & \cellcolor{tabthird}29.4 & 0.78 & 29.8 & \cellcolor{tabsecond}0.88 & \cellcolor{tabsecond}32.5 \\
EraseFlow & 0.89 & 26.1 & 0.35 & 25.8 & 0.71 & 30.1 \\
Meta & \cellcolor{tabthird}0.93 & 25.6 & 0.13 & 22.8 & 0.48 & 28.2 \\
SHS & \cellcolor{tabsecond}0.97 & 23.9 & 0.26 & 25.8 & 0.48 & 27.3 \\
EDiff & 0.49 & \cellcolor{tabsecond}29.6 & \cellcolor{tabsecond}0.88 & \cellcolor{tabsecond}31.2 & \cellcolor{tabthird}0.87 & \cellcolor{tabsecond}32.5 \\
\midrule
\textbf{\textsc{TILDE}} & 0.82 & 28.9 & \cellcolor{tabthird}0.80 & 30.4 & 0.82 & 31.8 \\
\bottomrule
\end{tabular}%
}
\end{table}

\begin{table}[htbp]
\centering
\caption{Concept unlearning results for \textbf{Cat}.}
\label{tab:concept_cat}
\vspace{0.3em}
\setlength{\tabcolsep}{4pt}
\resizebox{\textwidth}{!}{%
\begin{tabular}{@{}lcccccc@{}}
\toprule
\textbf{Method} & $U_{\mathrm{Acc}}\uparrow$ & $U_{\mathrm{CLIP}}\uparrow$ & $RR_{\mathrm{Acc}}\uparrow$ & $RR_{\mathrm{CLIP}}\uparrow$ & $GR_{\mathrm{Acc}}\uparrow$ & $GR_{\mathrm{CLIP}}\uparrow$ \\
\midrule
ESD-u & 0.49 & 28.7 & 0.77 & 31.0 & 0.80 & 31.8 \\
ESD-x & 0.61 & 27.8 & 0.77 & 30.6 & 0.77 & 31.3 \\
UCE & 0.27 & \cellcolor{tabsecond}31.2 & \cellcolor{tabfirst}0.99 & \cellcolor{tabfirst}32.6 & \cellcolor{tabsecond}0.88 & \cellcolor{tabsecond}32.6 \\
CA & 0.33 & \cellcolor{tabfirst}31.3 & \cellcolor{tabfirst}0.99 & \cellcolor{tabfirst}32.6 & \cellcolor{tabfirst}0.89 & \cellcolor{tabfirst}32.9 \\
MACE & 0.12 & \cellcolor{tabsecond}31.2 & \cellcolor{tabthird}0.95 & \cellcolor{tabthird}32.2 & 0.85 & 32.4 \\
DUO & \cellcolor{tabfirst}0.96 & 27.9 & 0.85 & 32.1 & \cellcolor{tabthird}0.87 & \cellcolor{tabsecond}32.6 \\
EraseFlow & \cellcolor{tabthird}0.94 & 25.4 & 0.27 & 27.0 & 0.53 & 28.7 \\
Meta & \cellcolor{tabsecond}0.95 & 24.4 & 0.46 & 27.5 & 0.72 & 30.3 \\
SHS & 0.88 & 18.8 & 0.11 & 21.1 & 0.16 & 22.8 \\
EDiff & 0.18 & \cellcolor{tabthird}30.9 & \cellcolor{tabsecond}0.98 & \cellcolor{tabsecond}32.4 & \cellcolor{tabthird}0.87 & \cellcolor{tabthird}32.5 \\
\midrule
\textbf{\textsc{TILDE}} & \cellcolor{tabfirst}0.96 & 26.3 & 0.82 & 30.9 & \cellcolor{tabthird}0.87 & 32.1 \\
\bottomrule
\end{tabular}%
}
\end{table}

\begin{table}[htbp]
\centering
\caption{Concept unlearning results for \textbf{Dog}.}
\label{tab:concept_dog}
\vspace{0.3em}
\setlength{\tabcolsep}{4pt}
\resizebox{\textwidth}{!}{%
\begin{tabular}{@{}lcccccc@{}}
\toprule
\textbf{Method} & $U_{\mathrm{Acc}}\uparrow$ & $U_{\mathrm{CLIP}}\uparrow$ & $RR_{\mathrm{Acc}}\uparrow$ & $RR_{\mathrm{CLIP}}\uparrow$ & $GR_{\mathrm{Acc}}\uparrow$ & $GR_{\mathrm{CLIP}}\uparrow$ \\
\midrule
ESD-u & 0.39 & 28.9 & 0.81 & 31.7 & 0.82 & 31.8 \\
ESD-x & 0.46 & 27.8 & 0.77 & 31.0 & 0.79 & 31.3 \\
UCE & \cellcolor{tabthird}0.84 & 25.9 & \cellcolor{tabsecond}0.96 & 30.2 & 0.73 & 29.5 \\
CA & 0.23 & \cellcolor{tabfirst}30.8 & \cellcolor{tabfirst}1.00 & \cellcolor{tabfirst}33.2 & \cellcolor{tabfirst}0.91 & \cellcolor{tabfirst}32.8 \\
MACE & 0.21 & \cellcolor{tabsecond}30.0 & 0.94 & \cellcolor{tabsecond}32.8 & 0.86 & \cellcolor{tabthird}32.4 \\
DUO & 0.81 & 27.5 & 0.76 & 31.6 & \cellcolor{tabthird}0.87 & \cellcolor{tabsecond}32.5 \\
EraseFlow & 0.61 & 27.6 & 0.60 & 29.8 & 0.68 & 30.4 \\
Meta & \cellcolor{tabsecond}0.90 & 25.3 & 0.52 & 28.0 & 0.75 & 30.4 \\
SHS & \cellcolor{tabfirst}0.94 & 21.5 & 0.26 & 22.1 & 0.21 & 23.7 \\
EDiff & 0.12 & \cellcolor{tabthird}29.6 & \cellcolor{tabthird}0.95 & \cellcolor{tabthird}32.6 & \cellcolor{tabsecond}0.88 & \cellcolor{tabthird}32.4 \\
\midrule
\textbf{\textsc{TILDE}} & \cellcolor{tabfirst}0.94 & 21.8 & 0.59 & 27.4 & 0.80 & 31.2 \\
\bottomrule
\end{tabular}%
}
\end{table}

\begin{table}[htbp]
\centering
\caption{Concept unlearning results for \textbf{Golf Ball}.}
\label{tab:concept_golf_ball}
\vspace{0.3em}
\setlength{\tabcolsep}{4pt}
\resizebox{\textwidth}{!}{%
\begin{tabular}{@{}lcccccc@{}}
\toprule
\textbf{Method} & $U_{\mathrm{Acc}}\uparrow$ & $U_{\mathrm{CLIP}}\uparrow$ & $RR_{\mathrm{Acc}}\uparrow$ & $RR_{\mathrm{CLIP}}\uparrow$ & $GR_{\mathrm{Acc}}\uparrow$ & $GR_{\mathrm{CLIP}}\uparrow$ \\
\midrule
ESD-u & \cellcolor{tabthird}0.97 & 27.5 & 0.26 & 28.3 & 0.79 & 31.5 \\
ESD-x & \cellcolor{tabfirst}1.00 & 24.4 & 0.13 & 26.5 & 0.75 & 30.9 \\
UCE & 0.30 & \cellcolor{tabfirst}32.9 & \cellcolor{tabthird}0.65 & \cellcolor{tabfirst}30.5 & \cellcolor{tabsecond}0.88 & \cellcolor{tabthird}32.6 \\
CA & 0.92 & \cellcolor{tabsecond}31.7 & \cellcolor{tabsecond}0.66 & \cellcolor{tabsecond}30.4 & \cellcolor{tabfirst}0.89 & \cellcolor{tabfirst}32.8 \\
MACE & 0.85 & 29.4 & 0.48 & 29.7 & \cellcolor{tabthird}0.83 & 32.2 \\
DUO & 0.77 & \cellcolor{tabthird}31.4 & 0.48 & \cellcolor{tabsecond}30.4 & \cellcolor{tabfirst}0.89 & \cellcolor{tabsecond}32.7 \\
EraseFlow & 0.96 & 29.1 & 0.33 & 28.9 & 0.75 & 31.1 \\
Meta & \cellcolor{tabfirst}1.00 & 24.3 & 0.09 & 25.4 & 0.70 & 30.2 \\
SHS & 0.83 & 23.5 & 0.09 & 23.3 & 0.35 & 25.6 \\
EDiff & 0.48 & 30.1 & 0.53 & \cellcolor{tabthird}29.8 & \cellcolor{tabfirst}0.89 & \cellcolor{tabthird}32.6 \\
\midrule
\textbf{\textsc{TILDE}} & \cellcolor{tabsecond}0.99 & 23.4 & \cellcolor{tabfirst}0.67 & 28.4 & \cellcolor{tabsecond}0.88 & 31.9 \\
\bottomrule
\end{tabular}%
}
\end{table}

\begin{table}[htbp]
\centering
\caption{Concept unlearning results for \textbf{Lionel Messi}.}
\label{tab:concept_lionel_messi}
\vspace{0.3em}
\setlength{\tabcolsep}{4pt}
\resizebox{\textwidth}{!}{%
\begin{tabular}{@{}lcccccc@{}}
\toprule
\textbf{Method} & $U_{\mathrm{Acc}}\uparrow$ & $U_{\mathrm{CLIP}}\uparrow$ & $RR_{\mathrm{Acc}}\uparrow$ & $RR_{\mathrm{CLIP}}\uparrow$ & $GR_{\mathrm{Acc}}\uparrow$ & $GR_{\mathrm{CLIP}}\uparrow$ \\
\midrule
ESD-u & \cellcolor{tabfirst}1.00 & 25.9 & 0.31 & 30.4 & 0.77 & 31.7 \\
ESD-x & \cellcolor{tabsecond}0.99 & 24.9 & 0.26 & 28.7 & 0.70 & 30.6 \\
UCE & \cellcolor{tabfirst}1.00 & 22.5 & 0.25 & 26.0 & 0.49 & 26.7 \\
CA & \cellcolor{tabthird}0.95 & \cellcolor{tabthird}26.8 & \cellcolor{tabthird}0.61 & \cellcolor{tabthird}31.7 & \cellcolor{tabthird}0.87 & \cellcolor{tabthird}32.2 \\
MACE & \cellcolor{tabfirst}1.00 & 16.5 & 0.00 & 17.2 & 0.01 & 19.4 \\
DUO & \cellcolor{tabfirst}1.00 & \cellcolor{tabsecond}27.0 & 0.50 & \cellcolor{tabthird}31.7 & \cellcolor{tabthird}0.87 & \cellcolor{tabfirst}32.6 \\
EraseFlow & \cellcolor{tabsecond}0.99 & 21.3 & 0.08 & 24.3 & 0.45 & 26.9 \\
Meta & \cellcolor{tabfirst}1.00 & 23.9 & 0.19 & 28.3 & 0.74 & 31.1 \\
SHS & 0.94 & 19.4 & 0.12 & 21.5 & 0.26 & 23.5 \\
EDiff & 0.43 & \cellcolor{tabfirst}30.0 & \cellcolor{tabfirst}0.72 & \cellcolor{tabsecond}32.2 & \cellcolor{tabfirst}0.89 & \cellcolor{tabsecond}32.5 \\
\midrule
\textbf{\textsc{TILDE}} & \cellcolor{tabsecond}0.99 & 25.8 & \cellcolor{tabsecond}0.69 & \cellcolor{tabfirst}32.3 & \cellcolor{tabsecond}0.88 & \cellcolor{tabsecond}32.5 \\
\bottomrule
\end{tabular}%
}
\end{table}

\begin{table}[htbp]
\centering
\caption{Concept unlearning results for \textbf{Mickey Mouse}.}
\label{tab:concept_mickey_mouse}
\vspace{0.3em}
\setlength{\tabcolsep}{4pt}
\resizebox{\textwidth}{!}{%
\begin{tabular}{@{}lcccccc@{}}
\toprule
\textbf{Method} & $U_{\mathrm{Acc}}\uparrow$ & $U_{\mathrm{CLIP}}\uparrow$ & $RR_{\mathrm{Acc}}\uparrow$ & $RR_{\mathrm{CLIP}}\uparrow$ & $GR_{\mathrm{Acc}}\uparrow$ & $GR_{\mathrm{CLIP}}\uparrow$ \\
\midrule
ESD-u & \cellcolor{tabthird}0.97 & 28.0 & 0.21 & 30.7 & \cellcolor{tabthird}0.75 & 31.5 \\
ESD-x & \cellcolor{tabsecond}0.99 & 26.7 & 0.17 & 29.8 & 0.72 & 30.9 \\
UCE & 0.87 & \cellcolor{tabsecond}30.3 & \cellcolor{tabthird}0.46 & \cellcolor{tabthird}32.0 & \cellcolor{tabsecond}0.85 & 31.6 \\
CA & 0.96 & 29.6 & 0.32 & \cellcolor{tabsecond}32.4 & \cellcolor{tabfirst}0.87 & \cellcolor{tabfirst}32.6 \\
MACE & \cellcolor{tabfirst}1.00 & 20.8 & 0.00 & 17.0 & 0.02 & 19.5 \\
DUO & 0.93 & \cellcolor{tabthird}29.8 & 0.43 & \cellcolor{tabsecond}32.4 & \cellcolor{tabsecond}0.85 & \cellcolor{tabsecond}32.4 \\
EraseFlow & \cellcolor{tabfirst}1.00 & 25.0 & 0.09 & 27.5 & 0.60 & 29.3 \\
Meta & \cellcolor{tabsecond}0.99 & 26.3 & 0.07 & 27.9 & 0.67 & 30.4 \\
SHS & 0.88 & 22.1 & 0.35 & 25.5 & 0.33 & 24.7 \\
EDiff & 0.53 & \cellcolor{tabfirst}30.7 & \cellcolor{tabfirst}0.75 & \cellcolor{tabfirst}34.1 & \cellcolor{tabfirst}0.87 & \cellcolor{tabthird}32.3 \\
\midrule
\textbf{\textsc{TILDE}} & \cellcolor{tabthird}0.97 & 28.7 & \cellcolor{tabsecond}0.56 & 31.8 & \cellcolor{tabsecond}0.85 & 31.9 \\
\bottomrule
\end{tabular}%
}
\end{table}

\begin{table}[htbp]
\centering
\caption{Concept unlearning results for \textbf{Pikachu}.}
\label{tab:concept_pikachu}
\vspace{0.3em}
\setlength{\tabcolsep}{4pt}
\resizebox{\textwidth}{!}{%
\begin{tabular}{@{}lcccccc@{}}
\toprule
\textbf{Method} & $U_{\mathrm{Acc}}\uparrow$ & $U_{\mathrm{CLIP}}\uparrow$ & $RR_{\mathrm{Acc}}\uparrow$ & $RR_{\mathrm{CLIP}}\uparrow$ & $GR_{\mathrm{Acc}}\uparrow$ & $GR_{\mathrm{CLIP}}\uparrow$ \\
\midrule
ESD-u & \cellcolor{tabthird}0.96 & 27.9 & 0.50 & 30.8 & 0.75 & 31.5 \\
ESD-x & \cellcolor{tabfirst}1.00 & 0.0 & 0.00 & 0.0 & 0.00 & 0.0 \\
UCE & \cellcolor{tabfirst}1.00 & 21.8 & 0.68 & 27.0 & 0.47 & 26.1 \\
CA & \cellcolor{tabsecond}0.97 & 27.9 & \cellcolor{tabthird}0.74 & \cellcolor{tabsecond}33.8 & \cellcolor{tabsecond}0.87 & \cellcolor{tabfirst}32.4 \\
MACE & \cellcolor{tabthird}0.96 & 27.9 & 0.66 & \cellcolor{tabthird}33.6 & \cellcolor{tabthird}0.85 & \cellcolor{tabsecond}32.3 \\
DUO & \cellcolor{tabthird}0.96 & \cellcolor{tabfirst}30.9 & 0.63 & 33.2 & \cellcolor{tabthird}0.85 & \cellcolor{tabsecond}32.3 \\
EraseFlow & \cellcolor{tabsecond}0.97 & 26.6 & 0.42 & 29.7 & 0.64 & 30.2 \\
Meta & \cellcolor{tabsecond}0.97 & 26.2 & 0.45 & 29.2 & 0.69 & 30.0 \\
SHS & 0.83 & 23.6 & 0.38 & 25.1 & 0.27 & 24.3 \\
EDiff & 0.72 & \cellcolor{tabsecond}30.8 & \cellcolor{tabfirst}0.78 & \cellcolor{tabfirst}34.2 & \cellcolor{tabsecond}0.87 & \cellcolor{tabthird}32.2 \\
\midrule
\textbf{\textsc{TILDE}} & \cellcolor{tabfirst}1.00 & \cellcolor{tabthird}29.9 & \cellcolor{tabsecond}0.76 & 33.0 & \cellcolor{tabfirst}0.88 & 32.1 \\
\bottomrule
\end{tabular}%
}
\end{table}

\begin{table}[htbp]
\centering
\caption{Concept unlearning results for \textbf{Van Gogh Style}.}
\label{tab:concept_van_gogh_style}
\vspace{0.3em}
\setlength{\tabcolsep}{4pt}
\resizebox{\textwidth}{!}{%
\begin{tabular}{@{}lcccccc@{}}
\toprule
\textbf{Method} & $U_{\mathrm{Acc}}\uparrow$ & $U_{\mathrm{CLIP}}\uparrow$ & $RR_{\mathrm{Acc}}\uparrow$ & $RR_{\mathrm{CLIP}}\uparrow$ & $GR_{\mathrm{Acc}}\uparrow$ & $GR_{\mathrm{CLIP}}\uparrow$ \\
\midrule
ESD-u & 0.76 & \cellcolor{tabthird}30.1 & \cellcolor{tabthird}0.78 & \cellcolor{tabsecond}32.9 & 0.84 & \cellcolor{tabthird}32.1 \\
ESD-x & 0.81 & 29.4 & 0.75 & 32.2 & 0.81 & 31.7 \\
UCE & 0.49 & \cellcolor{tabfirst}32.5 & \cellcolor{tabfirst}0.91 & \cellcolor{tabfirst}33.6 & \cellcolor{tabfirst}0.89 & \cellcolor{tabfirst}32.6 \\
CA & \cellcolor{tabsecond}0.95 & 27.6 & \cellcolor{tabsecond}0.81 & \cellcolor{tabthird}32.5 & 0.85 & \cellcolor{tabsecond}32.4 \\
MACE & \cellcolor{tabfirst}1.00 & 17.9 & 0.00 & 18.1 & 0.00 & 20.7 \\
DUO & 0.68 & 21.8 & 0.57 & 29.5 & 0.85 & \cellcolor{tabthird}32.1 \\
EraseFlow & \cellcolor{tabthird}0.84 & 29.2 & 0.61 & 30.9 & 0.71 & 30.4 \\
Meta & 0.43 & 20.2 & 0.43 & 24.2 & 0.49 & 25.8 \\
SHS & 0.68 & 20.9 & 0.00 & 21.4 & 0.04 & 19.7 \\
EDiff & 0.29 & \cellcolor{tabsecond}31.7 & \cellcolor{tabsecond}0.81 & 31.0 & \cellcolor{tabsecond}0.88 & 31.6 \\
\midrule
\textbf{\textsc{TILDE}} & \cellcolor{tabsecond}0.95 & 27.1 & 0.75 & 31.5 & \cellcolor{tabthird}0.86 & 31.4 \\
\bottomrule
\end{tabular}%
}
\end{table}

\begin{table}[htbp]
\centering
\caption{Concept unlearning results for \textbf{Monet Style}.}
\label{tab:concept_monet_style}
\vspace{0.3em}
\setlength{\tabcolsep}{4pt}
\resizebox{\textwidth}{!}{%
\begin{tabular}{@{}lcccccc@{}}
\toprule
\textbf{Method} & $U_{\mathrm{Acc}}\uparrow$ & $U_{\mathrm{CLIP}}\uparrow$ & $RR_{\mathrm{Acc}}\uparrow$ & $RR_{\mathrm{CLIP}}\uparrow$ & $GR_{\mathrm{Acc}}\uparrow$ & $GR_{\mathrm{CLIP}}\uparrow$ \\
\midrule
ESD-u & 0.17 & \cellcolor{tabthird}30.3 & \cellcolor{tabsecond}0.87 & 31.1 & \cellcolor{tabthird}0.84 & 31.1 \\
ESD-x & 0.41 & 29.5 & 0.57 & 31.0 & 0.81 & 31.6 \\
UCE & 0.12 & \cellcolor{tabsecond}32.0 & \cellcolor{tabfirst}0.93 & \cellcolor{tabfirst}33.5 & 0.82 & \cellcolor{tabthird}32.0 \\
CA & 0.66 & 28.3 & 0.56 & \cellcolor{tabthird}31.2 & \cellcolor{tabfirst}0.88 & \cellcolor{tabsecond}32.4 \\
MACE & \cellcolor{tabfirst}1.00 & 20.4 & 0.07 & 20.9 & 0.01 & 19.4 \\
DUO & 0.77 & 25.5 & 0.12 & 26.7 & \cellcolor{tabfirst}0.88 & \cellcolor{tabfirst}32.5 \\
EraseFlow & \cellcolor{tabthird}0.84 & 28.3 & 0.60 & 29.4 & 0.73 & 30.2 \\
Meta & 0.71 & 24.5 & 0.26 & 26.0 & 0.69 & 29.9 \\
SHS & 0.76 & 22.9 & 0.26 & 24.8 & 0.36 & 25.6 \\
EDiff & 0.18 & \cellcolor{tabfirst}32.5 & \cellcolor{tabthird}0.85 & \cellcolor{tabsecond}33.1 & \cellcolor{tabsecond}0.87 & \cellcolor{tabsecond}32.4 \\
\midrule
\textbf{\textsc{TILDE}} & \cellcolor{tabsecond}0.87 & 27.3 & 0.74 & 30.9 & 0.83 & 31.5 \\
\bottomrule
\end{tabular}%
}
\end{table}

\end{document}